\documentclass{article}

\usepackage{graphics} 

\usepackage{epsfig} 

\usepackage{mathptmx} 

\usepackage{times} 
 \usepackage{amsthm}
\usepackage{amsmath} 

\usepackage{amssymb}  

\usepackage{color}

\usepackage{algorithm}

\usepackage{algorithmic}

\newtheorem{thm}{Theorem}

\newtheorem{dfn}{Definition}

\newtheorem{prop}{Proposition}

\newtheorem{rem}{Remark}

\newcommand{\RR}{{\mathbb R}}

\newcommand{\norm}[1]{\lVert#1\rVert}

\newcommand{\Rg}{{\mathbb R^{\text{dim } \mathfrak g}}}

\newcommand{\dv}[2]{{\frac{\partial #1}{\partial #2}}}

\newcommand{\dotex}{{\frac{d}{dt}}}

\DeclareMathOperator{\Cov}{Cov}

\title{\LARGE \bf
An EKF-SLAM algorithm with consistency properties}

\author{Axel Barrau, Silv\`ere Bonnabel
\thanks{A. Barrau is with SAFRAN TECH, Groupe Safran, Rue des
Jeunes Bois - Chateaufort, 78772 Magny Les Hameaux CEDEX.  S. Bonnabel are with MINES ParisTech, PSL Research University, Centre for robotics, 60 Bd St Michel 75006 Paris, France
       {\tt\footnotesize [axel.barrau,silvere.bonnabel]@mines-paristech.fr}}%
}

\include{myLatex}

\graphicspath{{Figures/}}

\begin{document}

\maketitle

\thispagestyle{empty}

\pagestyle{empty}


\begin{abstract}In this paper we address the inconsistency of the EKF-based SLAM algorithm that stems from non-observability of  the origin and orientation of the global reference frame. We prove on the non-linear two-dimensional problem with point landmarks observed that this type of inconsistency is remedied using the Invariant EKF, a recently introduced variant of the EKF meant to account for the symmetries of the state space. Extensive Monte-Carlo runs illustrate the theoretical results.
\end{abstract}
\section{Introduction}

The problem of simultaneous localization and mapping (SLAM) has a rich history over the past two decades, which is too broad to cover here, see e.g. \cite{dissanayake-2001,durrant2006simultaneous}. The extended Kalman filter (EKF) based SLAM (the EKF-SLAM) has played an important historical role, and is still   used, notably for its ability to close loops thanks to the maintenance of correlations between remote landmarks. 

The fact that the EKF-SLAM is inconsistent (that is, it returns a covariance matrix that is too optimistic, see e.g., \cite{Bar-Shalom}, leading to inaccurate estimates) was early noticed \cite{julier2001counter} and has since been explained in various papers  \cite{castellanos2004limits,
bailey2006consistency,
huang2008analysis,
huang2010observability,
huang2007convergence,huang2011observability}. In the present paper we consider the inconsistency issues that stem from the fact that, as only relative measurements are available, the origin and orientation of the earth-fixed frame can never be correctly estimated, but the EKF-SLAM tends to ``think" it can estimate them as its output covariance matrix reflects an information gain in those directions of the state space. This lack of observability, and the poor ability of the EKF to handle it, is notably regarded as the root cause of inconsistency in \cite{huang2010observability,huang2011observability} (see also references therein). In the present paper we advocate the use of the Invariant (I)-EKF to prevent covariance reduction in directions of the state space where no information is available.


The Invariant extended Kalman filter (IEKF) is a novel methodology  introduced in \cite{bonnabel-cdc,bonnabel-martin-salaun-cdc09} that consists in slightly modifying the EKF equations to have them respect the geometrical structure of the problem. Reserved to systems defined on Lie groups, it has been mainly driven by applications to localization and guidance, where it appears as a slight modification of the multiplicative EKF (MEKF), widely known and used in the world of aeronautics. It has  been
proved to possess theoretical local convergence properties the EKF lacks in \cite{barrau2014invariant}, to be an improvement over the EKF in practice (see e.g., \cite{barczyk2013invariant,barczyk2015invariant,diemer2014invariant,martin2010generalized} and more recently \cite{barrau2014invariant} where the EKF is outperformed), and has been successfully implemented in  industrial  applications to navigation (see the patent \cite{brevet_alignement}).

In the present paper, we slightly generalize the IEKF framework, to make it capable to handle very general observations (such as range and bearing or bearing only observations), and we show how the derived IEKF-SLAM, a simple variant of the EKF-SLAM, allows remedying the inconsistency of EKF-SLAM stemming from the non-observability of the orientation and origin of the global frame.

\subsection{Links and differences with previous literature}

The issue of EKF-SLAM inconsistency has been the object of many papers,  see 
\cite{julier2001counter, castellanos2004limits,
bailey2006consistency,
huang2007convergence} to cite a few, where empirical evidence (through Monte-Carlo simulations) and theoretical explanations in various particular situations have been accumulated. In particular, the insights of \cite{bailey2006consistency,
huang2007convergence} have been that the orientation uncertainty is a key feature in the inconsistency. The article \cite{huang2007convergence}, in line with \cite{julier2001counter, castellanos2004limits,martinel,
bailey2006consistency},  also underlines the importance of the linearization process, as linearizing about the true trajectory solves the inconsistency issues, but is impossible to implement in practice as the true state is unknown. It derives a relationship that should hold between various Jacobians appearing in the EKF equations when they are evaluated at the current state estimate to ensure consistency. 

A little later, the works of G.P. Huang, A.I. Mourikis, and S. I. Roumeliotis  \cite{huang2008analysis,
huang2010observability,huang2011observability} have provided a sound theoretical analysis of the EKF-SLAM inconsistency as caused by the EKF inability to correctly reflect the three unobservable degrees of freedom (as an overall rotation and translation of the global reference frame leave all the measurements unchanged).  Indeed, the filter tends to erroneously acquire information along the directions spanned by those unobservable transformations. To remedy this problem, the above mentioned authors have proposed various solutions, the most advanced being the Observability Constrained (OC)-EKF. The idea is to  pick a linearization point that is such that the unobservable subspace ``seen" by the EKF system model is of appropriate dimension, while minimizing the expected errors of the linearization points. 

Our approach, that relies on the IEKF, provides an interesting alternative to the OC-EKF, based on a quite different route. Indeed, the rationale is to apply the EKF methodology, but using alternative estimation errors to the standard linear difference between the estimate and the true state. Any non-linear error that reflects a discrepancy between the true state and the estimate, necessarily defines a local frame around any point,  and the idea underlying the IEKF amounts  to write the Kalman Jacobians and covariances in this frame. We notice and prove here that an alternative nonlinear error  defines a local frame where the unobservable subspace is \emph{everywhere} spanned by the same vectors. Using this local frame at the current estimate to express Kalman's covariance matrix will be shown to ensure the unobservable subspace ``seen" by the EKF system model is \emph{automatically} of appropriate dimension. 

We thus obtain an EKF variant  which automatically comes with consistency properties. Moreover, we relate unobservability to the inverse of the covariance matrix (called information matrix) rather than on the covariance matrix itself, and we derive guarantees of information decrease over unobservable directions. Contrarily to the OC-EKF, and as in the standard EKF, we use here  the latest, and thus best, state estimate as the linearization point to compute the filter Jacobians. 

In a nutshell, whereas the key fact for the analysis of \cite{huang2010observability} is that the choice of the linearization point affects the observability properties of the linearized state error system of the EKF,  the key fact for our analysis is that the choice of the error variable has similar consequences. Theoretical results and simulations underline the relevance of the proposed approach. 

Robot-centric formulations such as   \cite{castellanos}, and later  \cite{Guerreiro,Lourenco} are promising attempts to tackle unobservability, but they unfortunately lack convenience as  the  position of all the landmarks must be revised during the propagation step, so that the landmarks' estimated position becomes in turn sensitive to the motion sensor's noise. They do not provably solve the observability issues considered in the present paper, and it can be noted the OC-EKF has demonstrated better experimental performance than the robocentric mapping filter,  in \cite{huang2010observability}. In particular, the very recent papers \cite{Guerreiro,Lourenco} propose to write the equations of the SLAM in the robot's frame under a constant velocity assumption. Using an output injection technique, those equations become linear, allowing to  prove global asymptotic convergence of any linear observer for the corresponding deterministic linear model. This is fundamentally a deterministic approach and property, and as the matrices appearing in the obtained linear model are functions of the observations, the behavior of the filter is not easy to anticipate in a noisy context: The observation noise thus corrupts the very propagation step of the filter.

Some recent papers also propose to improve consistency through local map joining, see \cite{Zhao} and references therein. Although   appealing, this approach is rather oriented towards large-scale maps, and requires the existence of local submaps.  But when using submap  joining  algorithm,  ``inconsistency  in  even  one
of the submaps, leads to an inconsistent global map" \cite{huang2008submap}. This approach  may thus prove complementary, if the  IEKF SLAM proposed in the present paper is used to build consistent submaps.  Note that, the IEKF SLAM can also  be readily combined with other measurements such as the GPS, whereas the submap approach is  tailored for  pure  SLAM. 

From a methodology viewpoint, it is worth noting our approach does  {not} bring to bear estimation errors written in a robot frame, as \cite{castellanos,Guerreiro,Lourenco,Zhao}. Although based on symmetries as well, the estimation errors we use  are slightly more complicated.

Finally, nonlinear optimization techniques have become popular for SLAM recently, see e.g., \cite{Dellaert} as one of the first papers. Links between our approach, and those novel methods are discussed in the paper's conclusion.

\subsection{Paper's  organization}

The paper is organized as follows. In Section \ref{Sec:1}, the standard EKF equations and EKF-SLAM algorithm are reviewed. In Section \ref{Sec:2}  we recall the problem  that neither the origin nor the  orientation of the global frame are observable, but the EKF-SLAM systematically tends to ``think" it observes them, which leads to inconsistency. In Section \ref{Sec:22} we introduce the IEKF-SLAM algorithm. In Section \ref{Sec:3}  we show how the linearized model of the IEKF always correctly captures the considered unobservable directions. In Section \ref{sect::tools} we derive a property of the covariance matrix output  by the filter that can be interpreted in terms of Fisher information. In Section \ref{Sec:4}   simulations support the theoretical results and illustrate the benefits of the proposed algorithm.   Finally,   the IEKF theory of \cite{barrau2014invariant} is briefly recapped in the appendix, and the IEKF SLAM shown to be an application of this theory indeed. The equations of the IEKF SLAM in 3D are then also derived applying the general theory.

\section{The EKF-SLAM algorithm}
\label{Sec:1}

\subsection{Statement of the general  standard EKF equations}\label{gene:sec}

Consider a general dynamical system in discrete time with state $X_n \in \mathbb{R}^N$ associated to a sequence of observations $(Y_n)_{n \geqslant 0}\in\RR^p$. The equations are as follows:
\begin{equation}
\label{eq::general_dynamical_system}
X_{n} = f(X_{n-1},u_n,w_n),
\end{equation}
\begin{equation}\label{eq::general_dynamical_system_output}
Y_n = h(X_n)+V_n,
\end{equation}
where $f$ is the function encoding the evolution of the system, $w_n$ is the process noise, $u_n$ an input, $h$ the observation function and $V_n$ the measurement noise.

The EKF propagates the estimate $\hat X_{n-1|n-1}$ obtained after the  observation $Y_{n-1}$, through the deterministic part of \eqref{eq::general_dynamical_system}:
\begin{equation}\label{eq::obs_gen}
\hat X_{n|n-1} = f(\hat X_{n-1|n-1},u_n,0)\end{equation}
 The update of $\hat X_{n |n-1}$ using the new observation $Y_n$ is based on the first-order approximation of the non-linear system \eqref{eq::general_dynamical_system}, \eqref{eq::general_dynamical_system_output} around the estimate $\hat X_n $, with respect to  the estimation errors $e_{n-1|n-1},e_{n|n-1}$ defined as:
\begin{align}\label{err:elf}
e_{n-1|n-1} = X_{n-1} - \hat X_{n|n-1},\quad e_{n|n-1} = X_n - \hat X_{n|n-1}
\end{align}
Using the Jacobians $F_n=\dv{f}{X}(\hat X_{n-1|n-1},u_n,0)$, $G_n=\dv{f}{w}(\hat X_{n-1|n-1},u_n,0)$, and $H_n=\dv{h}{X}(\hat X_{n|n-1})$, the combination of equations \eqref{eq::general_dynamical_system}, \eqref{eq::general_dynamical_system_output} and  \eqref{eq::obs_gen}  yields the following first-order expansion of the error system
\begin{align}
e_{n|n-1} &= F_n e_{n-1|n-1} + G_n  w_n,\label{def:A_linear}\\
Y_n - h(\hat X_{n|n-1}) &= H_n  e_{n|n-1} + V_n,\label{def:H_linear}
\end{align}
where the second order terms, that is, terms  of order $O\left(\norm{ e}^2,\norm{w}^2, \norm{e} \norm{w}\right) $ have been removed according to the standard  way the EKF handles non-additive noises in the model (see e.g., \cite{Stengel} p. 386). Using the linear Kalman equations with $F_n,G_n,H_n,$ the gain $K_n$ is computed, and letting $z_n=Y_n - h(\hat X_{n|n-1})$, an estimate $e_{n|n}=K_nz_n$ of the error $X_n-\hat X_{n|n-1}$ accounting for the observation $Y_n$ is computed, along with its covariance matrix $P_{n|n}$. The state is updated accordingly:
\begin{equation}
\label{eq::update_linear}
\hat X_{n|n} = \hat X_{n|n-1} +K_nz_n
\end{equation}
The detailed equations are recalled in Algorithm \ref{algo::EKF}. The assumption underlying the EKF is that through first-order approximations \emph{of the state error} evolution, the linear Kalman equations allow computing a Gaussian approximation of the error $e_{n|n}\sim\mathcal N(0,P_{n|n})$ after each measurement, yielding an approximation of the sought density $\mathbb P(X_{n}|Y_1,\cdots,Y_n)\approx \mathcal N( \hat X_{n|n},P_{n|n})$. However, the linearizations involved induce inevitable approximations that may lead the filter to inconsistencies and sometimes even divergence. 

\begin{algorithmic}
\begin{algorithm}
\caption{Extended Kalman Filter (EKF)}
\label{algo::EKF}
\STATE{Choose an initial uncertainty matrix $P_0$ and estimate $\hat X_0$}
\LOOP
\STATE Define $F_n, G_n$ and $H_n$ through                                                                                                                                                           \eqref{def:A_linear} and \eqref{def:H_linear}.
\STATE Define $Q_n$ as $\Cov(w_n)$ and $R_n$ as $\Cov(V_n)$.
\STATE \textbf{Propagation} 
\STATE $\hat X_{n|n-1} = f \left( \hat X_{n-1|n-1},u_n,0 \right)$
\STATE $P_{n|n-1} = F_n P_{n-1|n-1} F_n^T + G_n Q_n G_n^T$
\STATE \textbf{Update} 
\STATE $  z_n=Y_n- h \left( \hat X_{n|n-1} \right)$
\STATE $  S_n = H_n P_{n|n-1} H_n^T + R_n $,
 \STATE $
  K_n = P_{n|n-1} H_n^T S_n^{-1}$
\STATE $P_{n|n} = [I-K_n H_n]P_{n|n-1} $
\STATE{ $ \hat X_{n|n} = \hat X_{n|n-1} +K_n z_n$}
\ENDLOOP
\end{algorithm}
\end{algorithmic}

\subsection{The considered SLAM problem}
\label{sect::SLAM_problem}
For simplicity's sake let us focus on the standard ``steered" bicycle (or unicycle) model \cite{durrant1996autonomous}. The state is defined as:
\begin{equation}\label{state:def}
X_n = \left( \theta_n, x_n, p_n^1, \ldots, p_n^K \right),
\end{equation}
where $\theta_n \in \mathbb{R}$ denotes the heading, $x_n \in \RR^2$ the 2D position of the robot/vehicle, $p_n^j \in \RR^2$ the position of   unknown landmark $j$ (landmarks or synonymously features, constitute the map). The equations of the model are:
\begin{equation}
\label{eq::bicycle_dyn}
\begin{aligned}
 \theta_n & = \theta_{n-1} + \omega_n+w_n^{\omega}, \\
 x_n & = x_{n-1} + R(\theta_{n-1}) (v_n+ w_n^v), \\
 p_n^j & = p_{n-1}^j,\quad 1\leq j\leq K
\end{aligned}
\end{equation}
where $\omega_n \in \mathbb{R}$ denotes the odometry-based estimate of the heading variation of the vehicle, $v_n \in \mathbb{R}^2$ the odometry-based indication of relative shift, $w_n^{\omega}$ and $w_n^v$ their associated noises, and $R(\theta)$ is the matrix encoding a rotation of angle $\theta$:
$$
R(\theta) = \begin{pmatrix} \cos(\theta) & -\sin(\theta) \\ \sin(\theta) & \cos(\theta) \end{pmatrix}.
$$Note that a forward Euler discretization of the continuous time well-known unicycle equations   leads to $v_n \in \mathbb{R}^2$ having its second entry null. More sophisticated integration methods or models including side slip may yet lead to non-zero values of both entries of $v_n$ so we opt for a more general model with $v_n \in \mathbb{R}^2$. 
The covariance matrix of the noises will be denoted by
\begin{equation}
\label{eq::Qn}
Q_n = \Cov \begin{pmatrix} w_n^{\omega} \\ w_n^v\\0_{2K\times 1} \end{pmatrix} = \mathbb{E} \left(\begin{pmatrix} w_n^{\omega} \\ w_n^v \\0_{2K\times 1}\end{pmatrix} \begin{pmatrix} w_n^{\omega} \\ w_n^v \\0_{2K\times 1}\end{pmatrix}^T \right)\in\RR^{l\times l}
\end{equation}
with $l=3+2K$. A general landmark observation in the robot's frame reads:
\begin{equation}
\label{eq::bicycle_obs}
Y_n = \begin{pmatrix} 
\tilde h \left[ R(\theta_n)^T \left( p^1-x_n \right) \right]+V_n^1 \\
\vdots \\
\tilde h \left[ R(\theta_n)^T \left( p^K-x_n \right) \right]+V_n^K 
\end{pmatrix}
\end{equation}
where $Y_n \in \RR^{2 K}$ (or $ \RR^{ K}$ for monocular visual SLAM) is the observation of the features at time step $n$, and $V_n$ the observation noise, and $\tilde h$ is any function.  
\begin{rem}Only a subset of the features is actually observed at time $n$.  However, to simplify the exposure of the filters' equations, we systematically assume in the sequel that all features are observed.\end{rem} 
We let the output noise covariance matrix be
\begin{equation}
\label{eq::Rn}
R_n = \Cov \begin{pmatrix} V_n^1 \\ \vdots \\ V_n^K \end{pmatrix}.
\end{equation}
\begin{rem}
Note that, the observation model \eqref{eq::bicycle_obs} encompasses the usual range and bearing observations used in the SLAM problem 
by letting $\tilde h  \begin{pmatrix}
y_1 \\ y_2 \end{pmatrix} = \left(
\sqrt{y_1^2 + y_2^2} ,
\arctan 2 \left( y_2, y_1 \right)
\right)
$. If we choose instead the one dimensional observation $\tilde h \begin{pmatrix}
y_1 \\ y_2
\end{pmatrix} = 
\arctan 2 \left( y_2, y_1 \right)
$ we recover the 2D monocular SLAM measurement. Note also we do not provide any specific form for the noise in the output: this is because the properties we are about to prove are related to the observability and thus only depend on the deterministic part of the system, so they are in fact totally insensitive to the way the noise enters the system.
\end{rem}


\subsection{The EKF-SLAM algorithm}

We merely apply here the methodology of EKF to the SLAM problem described in Section \ref{sect::SLAM_problem}. The first-order expansions \eqref{def:A_linear}, \eqref{def:H_linear} applied to equations \eqref{eq::bicycle_dyn}, \eqref{eq::bicycle_obs} yield:
\begin{equation}
\begin{gathered}
\label{eq::F_H_linear_SLAM}
F_n = \begin{pmatrix} 1 & 0_{1,2} & 0_{1,2K} \\ 
                       R \left(\hat \theta_{n-1|n-1} \right) J v_n^T & I_2 & 0_{2,2K} \\
                      0_{2K,1} & 0_{2K,2} & I_{2K}
 \end{pmatrix},\\
G_n = \begin{pmatrix}
1 & 0_{1,2} & 0_{1,2K} \\
0_{2,1} & R \left(\hat \theta_{n-1|n-1} \right)  & 0_{2,2K} \\
 0_{2K,1} & 0_{2K,2} & 0_{2K,2K}
 \end{pmatrix}, ~
H_n =\begin{pmatrix} \nabla h^1 \cdot H_n^1 \\ \vdots \\ \nabla h^K \cdot H_n^K \end{pmatrix} \\H_n^k =\begin{pmatrix} - J R \left( \tilde\theta \right)^T  \left( \hat{p}_{n|n-1}^k-\hat{x}_{n|n-1} \right) & -R \left(\tilde\theta  \right)^T & R \left(\tilde\theta \right)^T \end{pmatrix},
\end{gathered}
\end{equation}
with $\tilde\theta=\hat{\theta}_{n|n-1}$, $J = \begin{pmatrix} 0 & -1 \\ 1 & 0 \end{pmatrix}$ and $\nabla h^k$ denotes the Jacobian of $\tilde h$ computed at $ R \left( \hat \theta_{n|n-1} \right)^T \left[ \hat p_{n|n-1}^k - \hat x_{n|n-1} \right] \in \RR^2$. The obtained EKF-SLAM algorithm is  recaped in  Algorithm \ref{algo::EKF_SLAM_linear}. 

\begin{algorithmic}
\begin{algorithm}
\caption{EKF SLAM}
\label{algo::EKF_SLAM_linear}
\STATE{Choose an initial uncertainty matrix $P_0$ and estimate $\hat X_0$}
\LOOP
\STATE Define $F_n, G_n$ and $H_n$ as in \eqref{eq::F_H_linear_SLAM}.
\STATE Define $Q_n$,  $R_n$ as in \eqref{eq::Qn} and \eqref{eq::Rn}.
\STATE \textbf{Propagation} 
\STATE{ $\hat \theta_{n|n-1}  = \hat \theta_{n-1|n-1} + \omega_n$}
\STATE{ $\hat x_{n|n-1}  = \hat x_{n-1|n-1} + R \left( \hat \theta_{n-1|n-1} \right) v_n$}
\STATE{ $\hat p_{n|n-1}^j  = \hat p_{n-1|n-1}^j$} for all $1\leq j\leq K$
\STATE $P_{n|n-1} = F_n P_{n-1|n-1} F_n^T + G_n Q_n G_n^T$
\STATE \textbf{Update} 
\STATE $  z_n=Y_n- \begin{pmatrix} 
\tilde h \left[ R(\hat \theta_{n|n-1})^T \left( \hat p^1_{n|n-1}- \hat x_{n|n-1} \right) \right] \\
\vdots \\
\tilde h \left[ R(\hat \theta_{n|n-1})^T \left( \hat p^K_{n|n-1}- \hat x_{n|n-1} \right) \right] 
\end{pmatrix}$
\STATE $  S_n = H_n P_{n|n-1} H_n^T + R_n $,
 \STATE $
  K_n = P_{n|n-1} H_n^T S_n^{-1}$
  \STATE $P_{n|n} = [I-K_n H_n]P_{n|n-1} $
\STATE{$  \hat X_{n|n} = \hat X_{n|n-1} + K_n z_n$ }
\ENDLOOP
\end{algorithm}
\end{algorithmic}

\section{Observability issues and consistency of the EKF}\label{Sec:2}
In this section we come back to the general framework \eqref{eq::general_dynamical_system}, \eqref{eq::general_dynamical_system_output}.
The standard issue of  observability \cite{gauthier1994observability} is fundamentally a deterministic notion so the noise is systematically  turned  off.

\begin{dfn}[Unobservable transformation]
\label{def::non_obs}
We say a transformation $\phi: \mathbb{R}^N \rightarrow \RR^N$ of the system \eqref{eq::general_dynamical_system}-\eqref{eq::general_dynamical_system_output} is unobservable if for any initial conditions $X_0^1\in\RR^N$ and $X_0^2 = \phi \left(X_0^1\right)$ the induced solutions of  the dynamics \eqref{eq::general_dynamical_system} with noise turned off, i.e., $ X_n=f(X_{n-1},u_n,0)$ yield the same output at each  time step $n \geqslant 0$, that is:
$$h(X_n^1)=h(X_n^2).$$
\end{dfn}
It concretely means that (with all noises turned off) if the transformation is applied to the initial state then none of the observations $Y_n$ are going to be affected. As a consequence, there is no way to know this transformation has been applied. In line with \cite{lee2006observability,huang2008analysis,
huang2010observability} we will focus here on the observability properties of the linearized system. To that end we define the notion of non-observable (or unobservable) shift which is an infinitesimal counterpart to Definition \ref{def::non_obs}, and is strongly related to the  infinitesimal observability \cite{gauthier1994observability}:

\begin{dfn}[Unobservable shift]
\label{def::non_obs_first_order}
Let $(X_n)_{n \geqslant 0}$ denote a solution of \eqref{eq::general_dynamical_system} with noise turned off. A vector $\delta X_0 \in \mathbb{R}^N$ is said to be an unobservable shift of \eqref{eq::general_dynamical_system}-\eqref{eq::general_dynamical_system_output} around $X_0$ if:
$$
\forall n \geqslant 0, \quad H_n \delta X_n = 0,
$$
where $H_n$ is the linearization of $h$ at $X_n$ and where $\delta X_n$ is the solution at $n$ of the linearized system  $\delta X_n = F_n \delta X_{n-1}$ initialized at $\delta X_0$,  with $F_n $ denoting the Jacobian matrix of $f(\cdot, u_n,0)$ computed at $ X_{n-1}$. 

In other words (see  e.g.  \cite{huang2010observability}), for all $n\geq 0$, $\delta X_n$  lies in the kernel of the  observability matrix   between steps $0$ and $n$ associated to the linearized error-state system model, i.e.,   $\delta X_0^T[H_0^T;(H_{1}F_1)^T;\cdots ;(H_{n}F_{n}\cdots F_1)^T]=0$.
\end{dfn}
The interpretation is as follows: consider another initial state   shifted from $X_0$ to $X_0 + \delta X_0$. Saying that $\delta X_0$ is unobservable means no  difference on the sequence of observations up to the first order could be detected between both trajectories. Formally, this condition reads: $h(X_n + \delta X_n)=h(X_n)+\circ \left( \delta X_n \right)$, i.e., $H_n \delta X_n=0$. An estimation method conveying its own estimation uncertainty as the EKF, albeit based on linearizations, should be able to detect such directions and to reflect that accurate estimates along such directions are beyond reach.

\subsection{Considered unobservable shifts}\label{direc:sec}

In the present paper we consider unobservability corresponding to the impossibility to observe the position and orientation of the global frame \cite{lee2006observability,huang2008analysis}. The corresponding shifts have already been derived in the literature. 

\begin{prop}
\label{prop::first_order_rotations_prelim}\cite{huang2010observability}
Let $\hat X = \begin{pmatrix}
\hat \theta , \hat x , \hat p
\end{pmatrix} $ be an estimate of the state. Only one feature is considered, the generalization of the proposition to several features is trivial. The first-order perturbation of the estimate corresponding to an infinitesimal rotation of angle $\delta \alpha$ of the global frame   consists of the shift $ \begin{pmatrix}
1 \\ J \hat x \\ J \hat p
\end{pmatrix} \delta \alpha,$ with $J=\begin{pmatrix} 0 & -1 \\ 1 & 0 \end{pmatrix}$. In the same way, the first-order perturbation of the estimate corresponding to an infinitesimal translation of the global frame of vector $\delta u\in\RR^2$ consists of the shift $\begin{pmatrix} 0 \\ \delta u \\ \delta u \end{pmatrix}$. 
\end{prop}
\begin{proof}
When rotating the global frame the heading becomes:
$$
\hat \theta \rightarrow \hat \theta + \delta \alpha.
$$
The position of the robot becomes:
$$
\hat x \rightarrow
\begin{pmatrix}
\cos(\delta \alpha) & -\sin(\delta \alpha) \\
\sin(\delta \alpha) & \cos(\delta \alpha)
\end{pmatrix} \hat x \approx \hat x+\delta \alpha \begin{pmatrix} 0 & -1 \\ 1 & 0 \end{pmatrix} \hat x.
$$
The position of the feature becomes:
$$
\hat p \rightarrow
\begin{pmatrix}
\cos(\delta \alpha) & -\sin(\delta \alpha) \\
\sin(\delta \alpha) & \cos(\delta \alpha)
\end{pmatrix} \hat p \approx \hat p +\delta \alpha \begin{pmatrix} 0 & -1 \\ 1 & 0 \end{pmatrix} \hat p.
$$
Stacking these results we obtain the first-order variation of the full state vector (regarding the rotation only, the effect of infinitesimal translation being trivial to derive):
\begin{align}\label{stacking}
\begin{pmatrix} \hat \theta \\ \hat x \\ \hat p \end{pmatrix}
\rightarrow \begin{pmatrix} \hat \theta \\ \hat x \\ \hat p \end{pmatrix} + 
\begin{pmatrix} 1 \\ J \hat x \\ J \hat p \end{pmatrix} \delta \alpha.
\end{align}
\end{proof}
\begin{prop}
\label{prop::first_order_rotations}\cite{huang2010observability}
The   shifts  of Proposition \ref{prop::first_order_rotations}  that correspond to infinitesimal rotations, are unobservable shifts of \eqref{eq::bicycle_dyn}-\eqref{eq::bicycle_obs} in the sense of Definition \ref{def::non_obs_first_order}.
\end{prop}
The intuitive explanation is clear \cite{huang2010observability}:  ``if the robot and landmark positions are shifted equally along those vectors, it will not be possible to distinguish the shifted position from the original one through the measurements.'' 

\subsection{Inconsistency of the EKF}
\label{sect::EKS_SLAM8inconsistency}

This  section recalls using the notations of the present paper, a result of \cite{huang2008analysis}. It shows the infinitesimal rotations defined in Proposition \ref{prop::first_order_rotations} are not, in general, unobservable shifts of the system linearized about the trajectory estimated by the EKF. Indeed, applying Definition \ref{def::non_obs_first_order} to \eqref{eq::F_H_linear_SLAM} in the case of a single feature (the generalization being straightforward) with $ \delta X_0 = ( 1 , J \hat x_{0|0},J \hat p_{0|0} )^T$ and  $\delta X_n =F_n\cdots F_1\delta X_0 $ yields  the condition for an infinitesimal rotation of the initial state to be unobservable for the linearized system. This condition writes $H_n\delta X_n\equiv0$ and boils down to have for any $n >0$ (see \cite{huang2008analysis}):
$$
\nabla h_n \cdot J \cdot R(\theta_{n|n-1})^T \left[ -  \left( \hat p_{n|n-1} - \hat p_{0|0} \right) + \sum_{i=1}^{n-1} \left( \hat x_{i|i} - \hat x_{i|i-1} \right) \right]=0
$$
where $ \nabla h_n$ is the Jacobian of $\tilde h$ computed at $R \left( \hat \theta_{n|n-1} \right)^T \left( \hat p_{n|n-1} - \hat x_{n|n-1} \right)$. For example, if $\tilde h$ is invertible the condition boils down to 
 \begin{align}
\forall n>0,~~ \left[ -  \left( \hat p_{n|n-1} - \hat p_{0|0} \right) + \sum_{i=1}^{n-1} \left( \hat x_{i|i} - \hat x_{i|i-1} \right) \right]=0.\label{never_verified}
\end{align}
We see the quantities involved are the updates of the state. As they depend on the noise, there is a null probability  for the condition to be respected, and it is always violated in practice.

But the point of the present paper is to show that the problem is related to the (arbitrary in a non-linear context) choice to represent the estimation error  as the linear difference $X - \hat X$, not to an inconsistency issue inherent to EKF-like methods applied to SLAM. By devising an EKF-SLAM based on another estimation error variable, which in some sense amounts to change coordinates, the false observability problem can be corrected.  
The qualitative reason why this is sufficient is related to the basic cause of false observability: a given fixed shift may or may not be observable depending on the linearization point $\hat X$, as proved by Proposition \ref{prop::first_order_rotations}.  It turns out that the latter property is not inherently related to the SLAM problem: it is in fact a mere consequence of the errors' definition \eqref{err:elf}. Defining those errors otherwise can dramatically modify the condition \eqref{never_verified}. This is the object of the remainder of this article.

\section{A novel EKF-SLAM algorithm}\label{Sec:22}

Building upon the theory of the Invariant (I)EKF on matrix Lie groups, as described and studied in \cite{barrau2014invariant}, we introduce in this section a novel IEKF for SLAM.  In Appendix \ref{primer}-\ref{gen:IEKF} the general theory of the IEKF is recalled and slightly extended to account for the very general form of output \eqref{eq::bicycle_obs}, and the algorithm derived herein is shown to be a direct application of the theory. To spare the reader a  study of the Lie group based theory, we attempt  to explain in simple terms  the IEKF methodology on the particular SLAM example throughout the present section.

Consider the model equations \eqref{eq::bicycle_dyn} with state $X_n$ given by \eqref{state:def}. Exactly as the EKF, the IEKF propagates the estimated state obtained after the observation $Y_{n-1}$ of \eqref{eq::bicycle_obs} through the deterministic part of  \eqref{eq::bicycle_dyn} i.e., $\hat \theta_{n|n-1}  = \hat \theta_{n-1|n-1} + \omega_n$, $\hat x_{n|n-1}  = \hat x_{n-1|n-1} + R \left( \hat \theta_{n-1|n-1} \right) v_n$, $\hat p_{n|n-1}^j  = \hat p_{n-1|n-1}^j$ for all $1\leq j\leq K$. 
To update the predicted state $\hat X_{n|n-1}$ using the observation $Y_n$ we use a first order Taylor expansion of the error system. But, \emph{instead of considering the usual state error  $X-\hat X$}, we rather use the (linearized) estimation  error  defined as follows 
\begin{equation}
\label{xi::non_linear_error}
\xi_{n|n-1} = \begin{pmatrix}
\theta_n - \hat \theta_{n|n-1} \\
  x_n - \hat x_{n|n-1} -\left( \theta_n -{\hat \theta}_{n|n-1} \right)J \hat x_n\\
p_n^1- \hat p^1_{n|n-1}-\left( \theta_n -{\hat \theta}_{n|n-1} \right)J \hat p^1_{n|n-1}
  \\
 \vdots \\
 p^K_n - \hat p^K_{n|n-1}-\left( \theta_n -{\hat \theta}_{n|n-1} \right)J \hat p^K_{n|n-1}
\end{pmatrix} 
\end{equation}and  $\xi_{n|n}$ is analogously defined. For close-by $X,\hat X$,   this represents an  error variable in the usual sense indeed, as $\xi=0$ if and only if $\hat X=X$. As in the standard EKF methodology, let us see how this \emph{alternative} estimation error is propagated through a first-order approximation of the error system. Using  the propagation equations of the filter, and \eqref{eq::bicycle_dyn}, we find 
\begin{equation}\begin{aligned}\label{first:linearized}\xi_{n|n-1}=\xi_{n-1|n-1}+\begin{pmatrix}w_n^\omega \\
    - w_n^\omega J \hat x_{n-1|n-1}+R({\hat \theta}_{n-1|n-1}) w_n^v \\
  - w_n^\omega J  \hat p^1_{n-1|n-1}
\\
 \vdots\\
   - w_n^\omega J  \hat p^K_{n-1|n-1} \end{pmatrix}
\end{aligned}\end{equation}
where terms of order $O(\norm{\xi_{n-1|n-1}}^2)$, $O( \norm{w_n^\omega}\norm{\xi_{n-1|n-1}})$, and $O( \norm{w_n^\omega}^2)$ have been neglected as in the standard   the EKF handles non-additive noises  \cite{Stengel}. To derive \eqref{first:linearized}  we have used the equalities $\forall \theta,\hat\theta\in\RR$, $w\in\RR^2$:
\begin{align}
& R(\theta)w=R(\hat\theta)w+O(|\hat\theta-\theta|~\norm{w})\label{K1}
\\&R(\theta)-R(\hat\theta)-(\theta-\hat\theta)JR(\hat\theta)=O(|\theta-\hat\theta|^2)\label{K2}
\end{align}
Note that, the odometer outputs $\omega_n,v_n$ have miraculously vanished. This is in fact a characteristics - and a key feature - of the IEKF approach. 

Let us now compute the first-order approximation of the observation error, using the alternative state error \eqref{xi::non_linear_error}. Define $H_n$ as the matrix, depending on $\hat X_{n|n-1}$ only, such that for all $\xi_{n|n-1} \in\RR^{2K+3}$   defined by  \eqref{xi::non_linear_error},  the innovation term
\begin{equation*}
\begin{pmatrix} 
\tilde h \left[ R(\theta_n)^T \left( p^1-x_n \right) \right] \\
\vdots \\
\tilde h \left[ R(\theta_n)^T \left( p^K-x_n \right) \right]
\end{pmatrix}-\begin{pmatrix} 
\tilde h \left[ R(\hat \theta_{n|n-1} )^T \left(\hat p^1_{n|n-1}-\hat x_{n|n-1} \right) \right] \\
\vdots \\
\tilde h \left[ R(\hat \theta_{n|n-1} )^T \left( \hat p^K_{n|n-1}-\hat x_{n|n-1} \right) \right]
\end{pmatrix}
\end{equation*}is equal to $H_n\xi_{n|n-1} +O(\norm{\xi_{n|n-1} }^2)$. Using that $R(\theta)^T(p-x)-R(\hat \theta)^T(\hat p-\hat x)=R(\theta)^T[(p-x)-R(\theta-\hat \theta)(\hat p-\hat x)]$, and $R(\theta)^T {\xi}=R(\hat \theta)^T {\xi}+O(\norm{\xi}^2)$  from \eqref{K1}, we see that  $H_n$ is defined as in \eqref{eq::A_H_G_IEKF} below. Thus the linearized (first-order) system  model with respect to alternative error  \eqref{xi::non_linear_error} writes
\begin{equation}
\begin{aligned}\label{def:A_non_linear_Lie_2}
\xi_{n|n-1}  &=F_n\xi_{n-1|n-1} +G_n  w_n,\\
Y_n-h(\hat X_{n|n-1})&=H_n\xi_{n|n-1} +V_n
\end{aligned}
\end{equation}
with $w_n^T=(w_n^\omega,(w_n^v)^T,0_{1\times 2K})^T$, and 
\begin{equation}
\label{eq::A_H_G_IEKF}
\begin{gathered}
F_n  = I_{2K+3}, ~
G_n  = \begin{pmatrix} 1 & 0_{1,2} & 0_{1,2K} \\
- J \hat x_{\tiny{n-1|n-1}} & R \left(\hat \theta_{n-1|n-1}\right) & 0_{2,2K} \\
 -J \hat p^1_{n-1|n-1} & 0_2 & 0_{2,2K} \\
\vdots & \vdots & \vdots \\
-J \hat p^K_{n-1|n-1} & 0_2 & 0_{2,2K}
 \end{pmatrix},\\
H_n  = \begin{pmatrix}
\nabla \tilde h^1 \cdot R \left( \hat \theta_{n|n-1} \right)^T \begin{pmatrix} 0_{2,1} & -I_2 & I_2 & 0_{2,2(K-1)} \end{pmatrix} \\
\nabla \tilde h^2 \cdot R \left( \hat \theta_{n|n-1} \right)^T \begin{pmatrix} 0_{2,1} & -I_2 & 0_{2,2} & I_2 & 0_{2,2(K-2)} \end{pmatrix} \\
\vdots \\
\nabla \tilde h^K \cdot R \left( \hat \theta_{n|n-1} \right)^T \begin{pmatrix} 0_{2,1} & -I_2 & 0_{2,2(K-1)} & I_2 \end{pmatrix}
\end{pmatrix},
\end{gathered}
\end{equation}
where $\nabla \tilde h^k$ is the Jacobian of $\tilde h$ computed at $R \left( \hat \theta_{n|n-1} \right)^T \left( \hat p_{n|n-1}^k - \hat x_{n|n-1} \right)$. 
As in the  standard EKF methodology, the matrices $F_n,G_n,H_n$ allow to compute the  Kalman gain $K_n$ and covariance $P_{n}$. Letting $z_n$ be the standardly defined innovation (see Algorithm \ref{algo::IEKF_SLAM} just after ``Update"),   $\xi_{n|n}=K_nz_n$ is an estimate of the linearized error $\xi_{n|n-1}$ accounting for the observation $Y_n$, and $P_{n|n}$ is supposed to encode the dispersion  $\mathbb E(\xi_{n|n}\xi_{n|n}^T)$. 

The final step of the standard EKF methodology is to update the estimated state $\hat X_{n|n-1}$ thanks to the estimated linearized error $\xi_{n|n}=K_nz_n$. There is a small catch, though: $\xi$ being not anymore defined as a mere difference $X-\hat X$, simply adding  $\xi_{n|n}$ to $\hat X_{n|n}$   would not be appropriate.  The most natural counterpart to \eqref{eq::update_linear} in our setting, would be to choose for $\hat X_{n|n}$ the values of $(\theta_n,x_n,\cdots, p_n^1,p_n^K)$ making the right member of \eqref{xi::non_linear_error}   equal to the just computed $\xi_{n|n}$. However, the  IEKF theory recalled in  Appendix \ref{gen:IEKF}, suggests an update that amounts to the latter to the first order, but whose non-linear structure ensures better properties \cite{barrau2014invariant}. Thus, the state is updated as follows $\hat X_{n|n}=\varphi(\xi_{n|n},\hat X_{n|n-1})=\varphi(K_nz_n,\hat X_{n|n-1})$, with $\varphi$   defined by
\begin{equation}
\begin{aligned}
\varphi\bigl(\begin{pmatrix}\delta \theta  \\  \delta x \\  \delta p^1 \\ \vdots \\  \delta p^K\end{pmatrix},\begin{pmatrix}\hat \theta  \\ \hat x \\  \hat p^1 \\ \cdots \\  \hat p^K\end{pmatrix}\bigr)= \begin{pmatrix}  \hat \theta +\delta {\theta} \\ R(\delta \theta) \hat x+ B(\delta \theta) \delta x \\ R(\delta \theta) \hat p^1 + B(\delta \theta) \delta p^1 \\ \vdots \\ R(\delta \theta) \hat p^K + B(\delta \theta) \delta p^K \end{pmatrix}\label{eq::exp_SE2}\end{aligned}
\end{equation}where $B(\alpha) = \begin{pmatrix} \frac{ \sin \left( \alpha \right) }{ \alpha } & - \frac{ 1 - \cos \left( \alpha \right) }{ \alpha } \\ \frac{1-\cos \left( \alpha \right) }{ \alpha } &  \frac{ \sin \left( \alpha \right) }{ \alpha } \end{pmatrix}$. Algorithm \ref{algo::IEKF_SLAM} recaps the various steps of the IEKF SLAM.

\begin{algorithmic}
\begin{algorithm}
\caption{IEKF SLAM}
\label{algo::IEKF_SLAM}
\STATE{The state is  defined by $X=(\theta,x^T,(p^1)^T,\cdots,(p^K)^T))\in\RR^{3+2K}$. Pick an initial uncertainty matrix $P_0$ and estimate $\hat X_0$.}
\LOOP
\STATE Define $F_n, G_n$ and $H_n$ as in \eqref{eq::A_H_G_IEKF}.
\STATE Define $Q_n$,  $R_n$ as in \eqref{eq::Qn} and \eqref{eq::Rn}.
\STATE \textbf{Propagation} 
\STATE{ $\hat \theta_{n|n-1}  = \hat \theta_{n-1|n-1} + \omega_n$}
\STATE{ $\hat x_{n|n-1}  = \hat x_{n-1|n-1} + R \left( \hat \theta_{n-1|n-1} \right) v_n$}
\STATE{ $\hat p_{n|n-1}^j  = \hat p_{n-1|n-1}^j$} for all $1\leq j\leq K$
\STATE $P_{n|n-1} = F_n P_{n-1|n-1} F_n^T + G_n Q_n G_n^T$
\STATE \textbf{Update} 
\STATE $  z_n=Y_n- \begin{pmatrix} 
\tilde h \left[ R(\hat \theta_{n|n-1})^T \left( \hat p^1_{n|n-1}- \hat x_{n|n-1} \right) \right] \\
\vdots \\
\tilde h \left[ R(\hat \theta_{n|n-1})^T \left( \hat p^K_{n|n-1}- \hat x_{n|n-1} \right) \right] 
\end{pmatrix}$
\STATE $  S_n = H_n P_{n|n-1} H_n^T + R_n $,
 \STATE $
  K_n = P_{n|n-1} H_n^T S_n^{-1}$
  \STATE $P_{n|n} = [I-K_n H_n]P_{n|n-1} $
\STATE{Use \eqref{eq::exp_SE2} to compute $ \hat X_{n|n}=\varphi(K_nz_n,\hat X_{n|n-1})$. }
\ENDLOOP
\end{algorithm}
\end{algorithmic}

\section{Remedying EKF SLAM  consistency}\label{Sec:3}

In this section we show  the infinitesimal rotations and translations of the global frame are unobservable shifts in the sense of  Definition \ref{def::non_obs_first_order} regardless of the linearization points  used to compute the matrices $F_n$ and $H_n$ of eq. \eqref{eq::A_H_G_IEKF}, a feature in sharp contrast with the usual restricting condition \eqref{never_verified} on the linearization points. In other words we show that infinitesimal rotations and translations of the global frame are always  unobservable shifts of the system model  \emph{linearized} with respect to error \eqref{xi::non_linear_error} regardless of the linearization point, a feature in sharp contrast with previous results (see Section \ref{sect::EKS_SLAM8inconsistency} and references therein). 

\subsection{Main result}

We can consider only one feature ($K=1$) without loss of generality. The expression of the linearized system model has become much simpler, as the linearized error  has the remarkable property to remain constant during the propagation step in the absence of noise, since $F_n = I_{3+2K}$ in \eqref{def:A_non_linear_Lie_2}-\eqref{eq::A_H_G_IEKF}. 
First, let us derive the impact of first-order variations  stemming from rotations and translations of the global frame on the error as  defined by   \eqref{xi::non_linear_error}, that is, an error of the following form
\begin{equation}
\label{eq::non_linear_error_f}
 \xi=\begin{pmatrix}(\theta -\hat\theta)\\x-\hat x- (\theta-\hat \theta) J\hat x\\p-\hat p- (\theta-\hat \theta) J\hat p\end{pmatrix}, 
\end{equation}
\begin{prop}
\label{prop::first_order_rotations_non_linear}
Let $\hat X = (
\hat \theta ,\hat x , \hat p )^T$ be an estimate of the state. The first-order perturbation of the \emph{linearized} estimation error  defined by \eqref{eq::non_linear_error_f} around 0, corresponding to an \emph{infinitesimal} rotation  of  angle $\delta \alpha$ of the global frame,  reads $ \begin{pmatrix}
1 \\ 0_{2,1} \\ 0_{2,1}
\end{pmatrix} \delta \alpha.$ In the same way, an  \emph{infinitesimal} translation of the global frame with vector $\delta u\in\RR^2$ implies a first-order perturbation of the error system \eqref{eq::non_linear_error_f}  of the form $(0,\delta u,\delta u)^T.$
\end{prop}
\begin{proof}
According to Proposition \ref{prop::first_order_rotations_prelim}, an infinitesimal rotation by an angle $\delta\alpha\ll 1$ of the true state corresponds to the transformation $\theta\to\theta+\delta\alpha$. $x\to x+\delta\alpha Jx$ and $p\to p+\delta\alpha Jp$. Regarding $\xi$ of eq \eqref{eq::non_linear_error_f} it corresponds to the variation
$$
\xi\to\xi+
\begin{pmatrix}
\delta \alpha \\
0_{2,1} \\
0_{2,1}
\end{pmatrix}+O(\delta\alpha\norm{\xi})+O(\delta\alpha^2).
$$This direction of the state space is ``seen" by the \emph{linearized} error system  as the vector $(\delta\alpha,0,0,0,0)^T$. Similarly,  a translation of vector $\delta u$ of the global frame yields the transformation $ \theta\to \theta,~  x\to\  x+\delta u,~ p\to p+\delta u$. The effect on the linearized error $\xi$ of \eqref{eq::non_linear_error_f}  is obviously the perturbation $(0,\delta u,\delta u)^T$ neglecting terms of order $\delta u\norm{\xi}$.
\end{proof}

We can now prove the first major result of the present article:  the infinitesimal transformations stemming from rotations and translations of the gobal frame  are unobservable shifts for the IEKF linearized model.

\begin{thm}
\label{SLAM:thm:obs}
Consider the SLAM problem defined by equations \eqref{eq::bicycle_dyn} and \eqref{eq::bicycle_obs}, and the IEKF-SLAM algorithm \ref{algo::IEKF_SLAM}. Let $\delta X_0$ denote a linear combination of  infinitesimal rotations and translations $\delta X_0^R, \delta X_0^1, \delta X_0^2$ of the whole system defined as follows
$$
\delta X_0^R = \begin{pmatrix}
1 \\
0_{2,1} \\
0_{2,1}
\end{pmatrix},
\qquad
\delta X_0^1 = \begin{pmatrix}
  0 \\
  1 \\
  0 \\
  1 \\
  0
\end{pmatrix},
\qquad
 \delta X_0^2 = \begin{pmatrix}
  0 \\
  0 \\
  1 \\
  0 \\
  1
\end{pmatrix}.
$$
Then $\delta X_0$ is an unobservable shift of the linearized system model \eqref{def:A_non_linear_Lie_2}-\eqref{eq::A_H_G_IEKF}  of the IEKF SLAM in the sense of Definition \ref{def::non_obs_first_order}, and this whatever the sequence of true states and estimates $(X_n,\hat X_{n|n},\hat X_{n|n-1})$: the very structure of the IEKF  is consistent with the considered unobservability.
\end{thm}
\begin{proof}
Note that Definition \ref{def::non_obs_first_order} involves a propagated perturbation $\delta X_n$, but as here $F_n$ is $I_5$: we have $\forall n>0, \delta X_n = \delta X_0$. Thus, the only point to check is:
$
H_n \left( \delta X_0 \right) = 0,
$
i.e., $\nabla \tilde h \cdot R \left( \hat \theta_{n|n-1} \right)^T \begin{pmatrix} 0_{2,1} & -I_2 & I_2 \end{pmatrix}\delta X_0= 0$. This is straightforward replacing $\delta X_0$ with alternatively $\delta X_0^1, \delta X_0^2$ and $\delta X_0^R$.
\end{proof}
We obtained the consistency property we were pursuing: the linearized model correctly captures the unobservability of global rotations and translations. As a byproduct, the unobservable seen by the filter is automatically of appropriate dimension. 

\subsection{Interpretation and discussion}\label{basis:sec}

The standard EKF is tuned to reduce   the state estimation error $\hat X-X$ defined through the original state variables $X,\hat X$ of the problem. Albeit perfectly suited to the linear case, the latter state error has in fact  absolutely no fundamental reason to rule the linearization process in a non-linear setting. The basic difference when analyzing the EKF and the IEKF is that
\begin{itemize}
\item In the standard EKF, there is a trivial correspondence between a small variation of the true state and a small variation of the estimation error \eqref{err:elf}. But the global rotations of the frame make the error vary in a non-trivial way as recalled in Section \ref{Sec:2}. 
\item In the IEKF approach, the effect of a small rotation of the state on the variation of the estimation error \eqref{xi::non_linear_error} becomes trivial as ensured by Proposition \ref{prop::first_order_rotations_non_linear}. But the error is non-trivially related to the state, as its definition explicitly depends on the linearization point $\hat X$.  
\end{itemize}

Many consistency issues of the EKF  stem from the fact that the updated covariance matrix $P_{n|n}$ is computed before the update,  namely at the predicted state $\hat X_{n|n-1}$, and thus does not account for the  updated state's value $\hat X_{n|n}$, albeit supposed to reflect the covariance of the updated error. This is why the OC-EKF typically seeks to avoid linearizing at the latest, albeit best, state estimate, in order to find a close-by state  such that the covariance matrix resulting from linearization preserves the  observability subspace dimension. The IEKF approach is wholly different: the updated covariance $P_{n|n}$  is  computed at the latest estimate $\hat X_{n|n-1}$, which is akin to the standard EKF methodology. But it is then indirectly adapted to the updated state, since it is \emph{interpreted} as the covariance of the error $\xi_{n|n}$. And contrarily to the standard case, the definition of this error   depends on  $\hat X_{n|n}$.   More intuitively, we can say the confidence ellipsoids encoded in $P_{n|n}$ are attached to a basis that undergoes a transformation when moved over from $\hat X_{n|n-1}$ to $\hat X_{n|n}$, this transformation being tied to the unobservable directions. This prevents spurious reduction of the covariance over unobservable shifts, which are not identical at $\hat X_{n|n-1}$ and  $\hat X_{n|n}$.  

Finally, note the alternative error \eqref{eq::non_linear_error_f}  is all but artificial: it naturally stems from the Lie group structure of the problem. This is logical as the considered unobservability actually pertains to an \emph{invariance} of the model \eqref{eq::bicycle_dyn}-\eqref{eq::bicycle_obs}, that is  the SLAM problem, to global translations and rotations. Thus it comes as no surprise the \emph{Invariant} approach, that brings to bear  {invariant} state errors that encode the very symmetries of the problem,  prove fruitful  (see the appendix for more details). 

\section{IEKF consistency and information}
\label{sect::tools}

Our approach can be related to the previous work \cite{huang2010observability}. Indeed, according to the latter article, failing to capture the right dimension of the observability subspace in the linearized model leads to ``spurious
information gain along directions of the state space where no
information is actually available'' and results in ``unjustified reduction of the covariance estimates,  a primary cause of filter inconsistency''. Theorem \ref{SLAM:thm:obs} proves that infinitesimal rotations and translations of the global frame, which are unobservable in the SLAM problem, are always ``seen'' by the IEKF linearized model as unobservable directions indeed, so this filter does not suffer from ``false observability'' issues. This is our major theoretical result.

That said,  the results of the latter section concern the system with noise turned off, and pertain to an automatic control approach to the notion of observability as in \cite{huang2010observability}. The present section is rather concerned with the estimation theoretic  consequences of Theorem \ref{SLAM:thm:obs}. We prove indeed, that the IEKF's output covariance matrix correctly reflects an absence of ``information gain''  along the unobservable directions, as mentioned above, but where  the information is now to be understood in the sense of Fisher information.
 As a by-product, this allows relating our results to a slightly different approach to SLAM consistency,   that rather  focuses on the Fisher information matrix  than on the observability matrix, see in particular \cite{Wang,Cetto}.

\subsection{The general Bayesian Fisher Information Matrix}\label{BFIM:sec}
The exposure of the present section is based on the seminal article \cite{Tichavsky}. See also \cite{Wang,Cetto} for related ideas applied to SLAM.   
Consider the system \eqref{eq::general_dynamical_system} with output \eqref{eq::general_dynamical_system_output}. Define the collection of state vectors and observations up to time $n$:
$$
\tilde X_n=(X_0^T,\cdots,X_n^T)^T,\quad \tilde Y_n=(Y_0^T,\cdots,Y_n^T)^T,
$$The joint probability distribution of the $(n+1)N $ vector $\tilde X_n$ and of the $np$ vector $\tilde Y_n$ is 
$$
p(\tilde Y_n,\tilde X_n)=p(X_0)\Pi_{i=1}^np(Y_i\mid X_i) p(X_i\mid X_{i-1})
$$The Bayesian Fisher Information Matrix (BIFM) is defined as the following $Nk\times Nk$ matrix based upon  the dyad of the gradient of the log-likelihood:
$$
J(\tilde X_n)=\mathbb E([\nabla_{\tilde X_n} \log p(\tilde Y_n,\tilde X_n)][\nabla_{\tilde X_n} \log p(\tilde Y_n,\tilde X_n]^T)
$$and note that, for the SLAM problem it boils down to the matrix of \cite{huang2010observability}. This matrix is of interest to us as it yields a lower bound on the accuracy achievable by  any estimator  used to attack the filtering problem \eqref{eq::general_dynamical_system}-\eqref{eq::general_dynamical_system_output}. Indeed let $J_n$ be defined as the \emph{inverse} of the $N\times N$ right-lower block of $[J(\tilde X_n)]^{-1}$. This matrix provides a lower bound on the mean square error of estimating $X_n$ from past and present measurements $\tilde Y_n$ and prior $p(X_0)$. Indeed, for any unbiased estimator $T:\RR^{np}\to \RR^N$:
$$
\mathbb E([T(\tilde Y_n)][T(\tilde Y_n)]^T)\succeq J_n^{-1}
$$
where $A\succeq B$ means $A-B$ is positive semi-definite. $J_n^{-1}$ is called the Bayesian or Posterior Cram\'er-Rao lower bound  for the filtering problem \cite{Tichavsky}. Most interestingly, in the case where $f$ and $h$ are linear, the prior distribution is Gaussian, and the noises are additive and Gaussian, we have 
$$
J_n=P_{n|n}^{-1}
$$where $P_{n|n}$ is the covariance matrix output by the Kalman filter. Thus, in the linear Gaussian case, $P_{n|n}^{-1}$  reflects  the statistical information available at time $n$ on the state $X_n$. By extension in the SLAM literature $P_{n|n}^{-1}$  is often simply referred to as the information matrix,  in non-linear contexts also, e.g. when using extended information filters  \cite{Thrun}.

\subsection{Application to  IEKF-SLAM consistency}

In the last section, we have recalled that in the linear Gaussian case, the inverse of the covariance matrix output by the Kalman filter {is} the Fisher information available to the filter (this is also stated in \cite{Bar-Shalom} p. 304). In the light of those results, it is natural to expect from any EKF variant, that the inverse of the output covariance matrix $P_{n|n}^{-1}$ reflect an absence of information gain along  unobservable directions indeed. If the filter fails to do so, the output covariance matrix will be too optimistic, that is, inconsistent, and wrong covariances yield wrong gains \cite{Bar-Shalom}. The following theorem shows the linearized system model of the IEKF allows ensuring the desired property of the covariance matrix. It is our second major result.

\begin{thm}
\label{SLAM:big:thm}
Consider the SLAM problem defined by equations \eqref{eq::bicycle_dyn}-\eqref{eq::bicycle_obs} and the IEKF-SLAM Algorithm \ref{algo::IEKF_SLAM}. Let $\delta X_0$ denote a linear combination of  infinitesimal rotations and translations $\delta X_0^R, \delta X_0^1, \delta X_0^2$ of the whole system, as defined in Theorem \ref{SLAM:thm:obs}.  $\delta X_0$ is thus an unobservable shift. If the matrix $P_{n|n}$  output by the IEKF remains invertible, we have at all times:
\begin{align}\left( \delta X_0 \right)^T P_{n|n}^{-1} \left( \delta X_0 \right) \leq \left( \delta X_0 \right)^T P_{n-1|n-1}^{-1} \left( \delta X_0 \right). \label{decrease:eq}
\end{align}
\end{thm}

\begin{proof}
As $F_n=I_{2K+3}$ in \eqref{eq::A_H_G_IEKF}, the unobservable shifts remain fixed i.e. $\delta X_n=\delta X_0$. At the propagation step we have:
\begin{align*}
\delta & X_0^T P_{n|n-1}^{-1}  \delta X_0  =   \delta X_{0}^T \left( P_{n-1|n-1}  + G_nQ_nG_n^T \right)^{-1} \delta X_{0} \\
\leqslant & \delta X_{0}^TP_{n-1|n-1}^{-1}\delta X_{0}.
\end{align*}
as $G_nQ_nG_n^T$ is positive semi-definite. And at the update step (see the Kalman Information Filter form in  \cite{Bar-Shalom}) we have:
$$
\delta X_0^T  P_{n|n}^{-1}  \delta X_0  = \delta X_0^T \left[ P_{n|n-1}^{-1} + H_n^T R_n^{-1} H_n \right] \delta X_0 
  = \delta X_0^T  P_{n|n-1}^{-1}  \delta X_0$$
as $H_n \delta X_0=0$ as shown in the proof of Theorem \ref{SLAM:thm:obs}. Thus $(\delta X_0)^T  P_{n|n}^{-1}  (\delta X_0)$ is non-increasing over time $n$. Note that, the proof evidences that if $P_{n|n}$  is not invertible, the results of the theorem still hold, writing the IEKF in information form. 
\end{proof}

Our result essentially means  the \emph{linearized model} of the IEKF has a structure which guarantees that the covariance matrix at all times reflects an absence of  ``spurious"  (Bayesian Fisher) information  gain  over directions that correspond to the unobservable rotations and translations of the global frame.

\section{Simulation results}
\label{Sec:4}

In this section we verify in simulation  the claimed properties on the one hand, and on the other we illustrate the striking consistency improvement  achieved by the IEKF SLAM. To that end, we propose to consider a similar numerical experiment as  in the sound work \cite{huang2010observability} dedicated to  the inconsistency of EKF and   the benefits of the OC-EKF. The IEKF is compared here to the standard EKF, the UKF, the OC-EKF, and the ideal EKF, which is the - impossible to implement - variant of the EKF where the state is linearized about the \emph{true} trajectory.

\begin{figure}
\begin{center}
\includegraphics[width=.60\columnwidth]{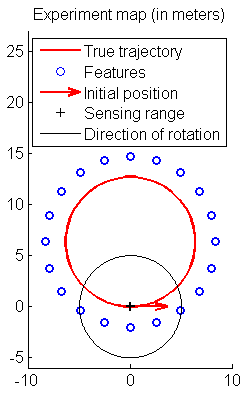}
\end{center}
\caption{ Simulated trajectory : the displayed loop is driven 10 times by a robot able to measure the relative position of the landmarks lying in a range of 5 m around him. Velocity is constant (1m/s) as well as angular velocity (9 deg/s).}
\label{fig::map}
\end{figure}

\begin{figure}[h]
\begin{center}
\includegraphics[width=\columnwidth]{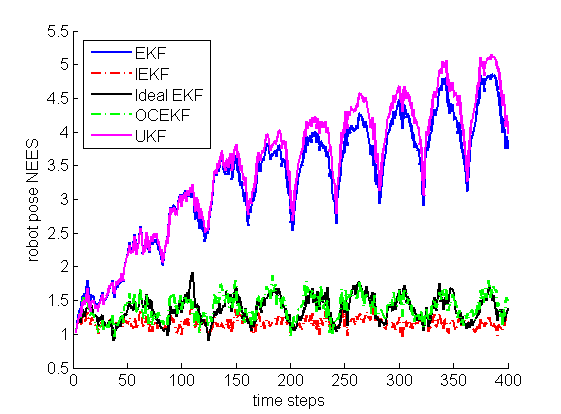}
\end{center}
\caption{ Evaluation of the consistency of the five filters using the NEES indicator (for the entire 3-DoF pose) over 50 runs, in the experimental setting described in Section \ref{sect::exp_setting}. We see the indicator stays around 1 for IEKF SLAM and OC-EKF SLAM over the whole time interval, as expected from a consistent estimation method. The ``ideal" EKF, where the system is linearized on the true value of the state, yields similar results. To the opposite, we see the EKF is inconsistent, and the UKF also.}
\label{fig::NEES}
\end{figure}

\begin{figure}
\begin{center}
\includegraphics[width=\columnwidth]{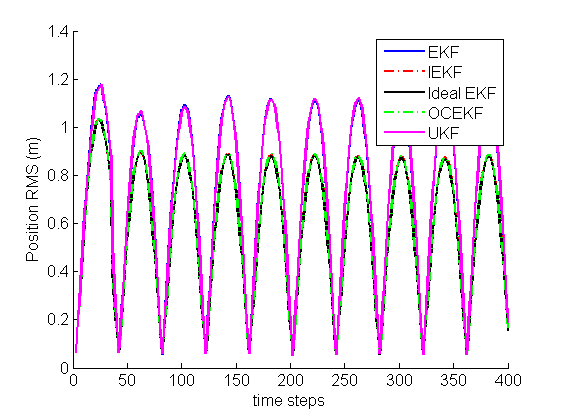}
\includegraphics[width=\columnwidth]{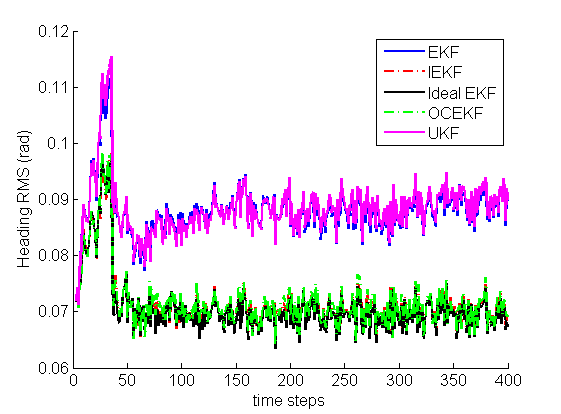}
\end{center}
\caption{
Evaluation of the performance of the proposed IEKF SLAM algorithm, in terms of RMS of the vehicle heading and position error (in rad and m respectively). We see the results are very similar to those of the OC-EKF and "ideal" EKF, the latter being impossible to implement for real (the system is linearized at the true state). These results are much better than those of the UKF and classical EKF.
}
\label{fig::heading_RMS}
\end{figure}

\begin{figure}
\includegraphics[width=.49 \columnwidth]{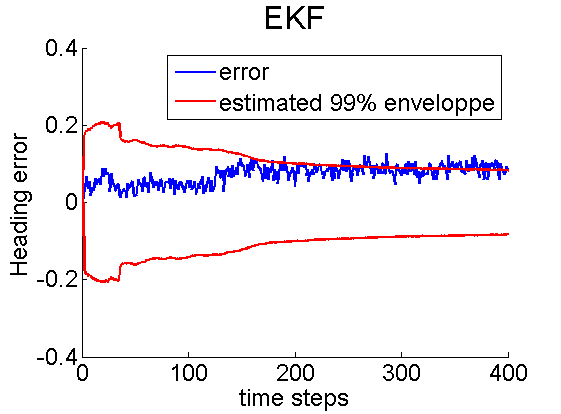}
\includegraphics[width=.49 \columnwidth]{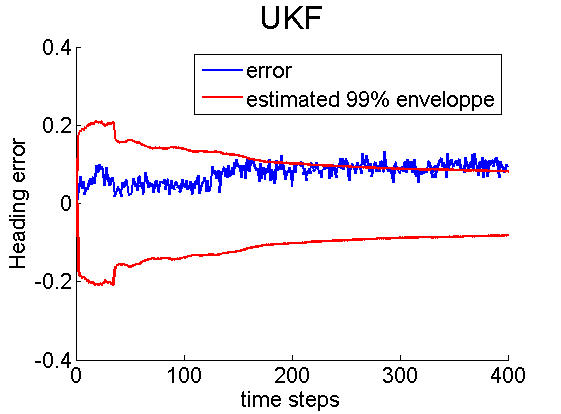}
\includegraphics[width=.49 \columnwidth]{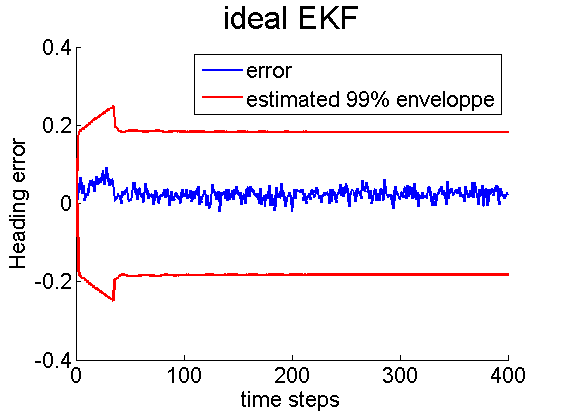}
\includegraphics[width=.49 \columnwidth]{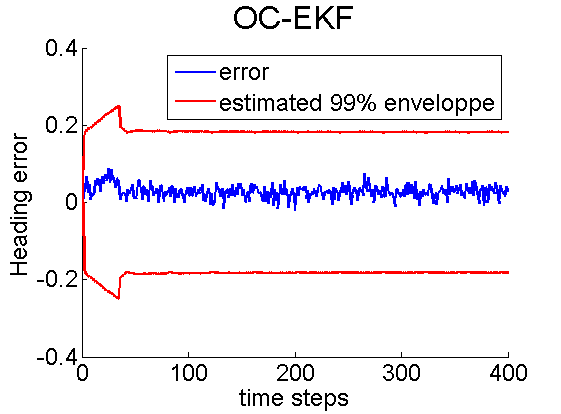}
\includegraphics[width=.49\columnwidth]{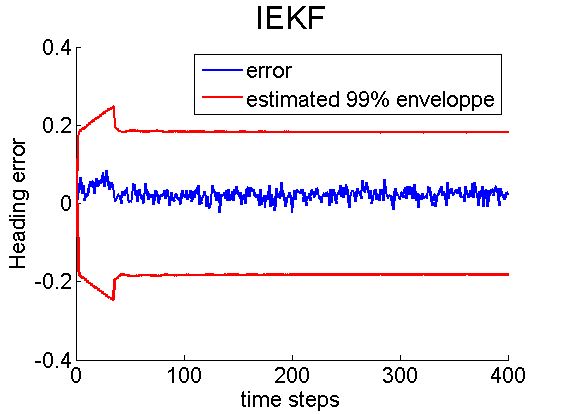}
\caption{Inconsistency   illustrated on a single run. The plotted EKF and UKF heading errors (top plots) do not remain in the $99 \%$ uncertainty envelope computed by the filter. Filters whith theoretical properties regarding non-observable directions (IEKF, OC-EKF and ideal EKF) remedy this problem.
}
\label{fig::EKF_enveloppe}
\end{figure}

\begin{figure}[h]
\includegraphics[width=.49 \columnwidth]{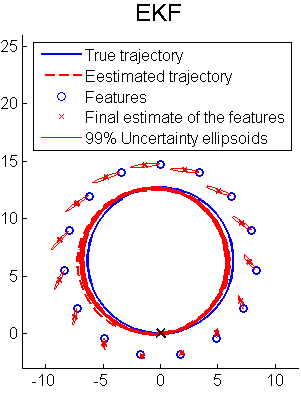}
\includegraphics[width=.49 \columnwidth]{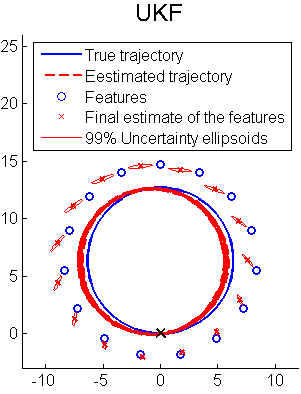}
\includegraphics[width=.49 \columnwidth]{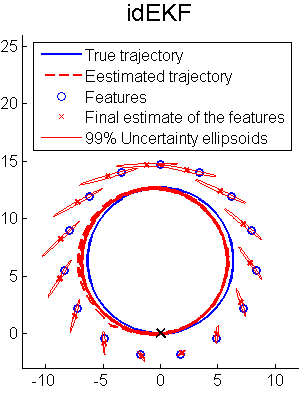}
\includegraphics[width=.49 \columnwidth]{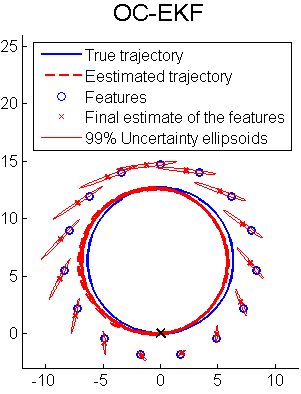}
\includegraphics[width=.49 \columnwidth]{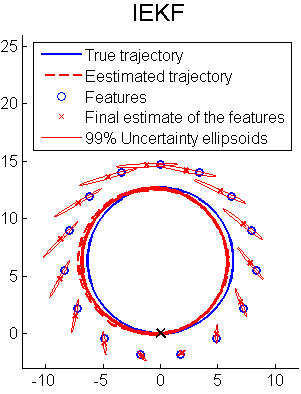}
\caption{ Trajectory and final covariance ellipsoids returned by the implemented filters. We see EKF and UKF are not consistent, mainly because the elongation of ellipsoids, that is related to heading uncertainty, is underestimated. It is not the case with ideal EKF (idEKF), OC-EKF and IEKF.
}
\label{fig::EKF_enveloppe2}
\end{figure}

\begin{figure}[h]
\begin{center}
\includegraphics[width=\columnwidth]{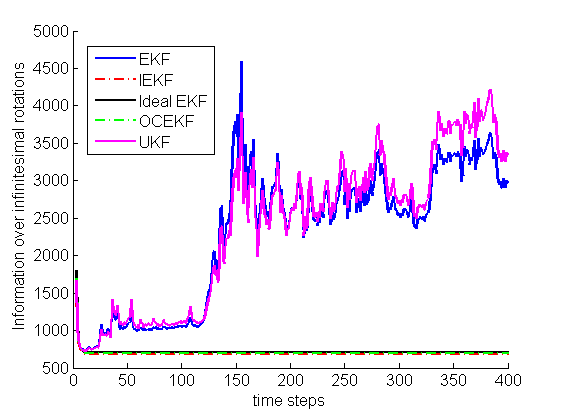}
\includegraphics[width=\columnwidth]{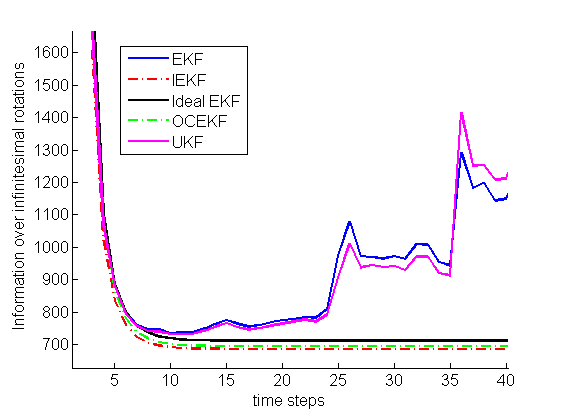}
\end{center}
\caption{Illustration of Theorem \ref{SLAM:big:thm}. Bottom plot is a zoom of the first time steps. The information over an infinitesimal perturbation corresponding to a rotation of the whole system is decreasing for the  IEKF SLAM, which is a consistent behavior as this perturbation is unobservable. Ideal EKF and OCEKF give similar results, but EKF and UKF do not. The plot also confirms EKF and UKF SLAM tend to acquire  spurious information over this unobservable direction.}
\label{fig::info_rot}
\end{figure}

\subsection{Simulation setting}
\label{sect::exp_setting}

The simulation setting we chose is (deliberately) similar to the one used in \cite{huang2010observability}  (Section 6.2). The vehicle (or robot) drives a 15m-diameter loop ten times in the 2D plane, finding on its path 20 unknown features as displayed on Figure \ref{fig::map}. The velocity and angular velocity are constant (1 m/s and 9 deg/s respectively). The relative position of the features in the reference frame of the vehicle is observed once every second (where $\tilde h$ of \eqref{eq::bicycle_obs} is the  identity). The standard deviation $\sigma$ of the velocity measurement on each wheel of the vehicle is $2 \%$ of the velocity. This yields the standard deviation $\sigma_v$ of the resulting linear velocity and $\sigma_{\omega}$ of the rotational velocity: $\sigma_v=(\sqrt 2/2)\sigma$ and $\sigma_{\omega} = (\sqrt 2/a) \sigma$, where $a=0.5m$ is the distance between the drive wheels (see \cite{huang2010observability}). The features are visible if they lie within a sensing range of 5 m, in which case they are observed with an isotropic noise of standard deviation 10 cm. The initial uncertainty over position and heading is zero -  which will prove a condition not sufficient to prevent failure of the EKF. Each time a landmark is seen for the first time, its position is initialized in the earth frame using the current estimated pose of the robot, the associated uncertainty is set to a very high value compared to the size of the map, then a Kalman update is performed to correlate the position of the new feature withe the other variables. Each second, all visible landmarks (i.e. those in a range of 5m) are processed simultaneously in a stacked observation vector.

Five algorithms are compared:
\begin{enumerate}
\item The classical EKF, described in Algorithm \ref{algo::EKF_SLAM_linear}.
\item The proposed IEKF SLAM algorithm  described in Algorithm \ref{algo::IEKF_SLAM}.
\item The ideal EKF as defined in \cite{huang2010observability}, i.e., a classical EKF where the Riccati equation  is computed at the true trajectory of the system instead of the estimated trajectory. Although not usable in practice, the latter is a good reference to compare with, as it is supposed to be  an EKF  with consistent behavior.
\item The OC-EKF described in \cite{huang2010observability}, which is so far the the only method that guarantees the non-observable subspace has appropriate dimension.
\item The Unscented Kalman Filter (UKF), known to  better deal with the non-linearities than the  EKF.
\end{enumerate}

Before going further, the next subsection introduces the NEES indicator used in the simulations to measure the consistency of these methods.

\subsection{The NEES indicator}
Classical criteria used to evaluate the performance of an estimation method, like Root Mean Squared (RMS) error do not inform about consistency as they do not take into account the uncertainty returned by the filter. This point is addressed by the Normalized Estimation Error Squared (NEES), which computes the average squared value of the error, normalized by the covariance matrix of the EKF. For a sample $(X_i)_{i=1,p}$ of error values having dimension $d$, each of them   with a covariance matrix $P_i$ of size $d\times d$, the NEES is defined by:
$$
\text{NEES} = \frac{1}{p\times d} \sum_{i=1}^p X_i^T P_i^{-1} X_i.
$$
If each $X_i$ is a zero mean Gaussian  with covariance matrix $P_i$, then for large $p$ we have NEES $\approx 1$. The case NEES $>1$ reveals an inconsistency   issue: the actual uncertainty is higher than the computed uncertainty. This situation typically occurs when the filter is optimistic as it believes to have gained information over a non-observable direction. The NEES indicator will be used, along with the usual RMS, to illustrate our solution to SLAM inconsistency in the sequel.

\subsection{Numerical results}
\label{sect::results}

Figure \ref{fig::NEES} displays the NEES indicator of the vehicle pose estimate (heading and position) over time, computed for 50 Monte-Carlo runs of the experiment described in Section \ref{sect::exp_setting}. As expected, the profile of the NEES for classical EKF, ideal EKF and OC-EKF is the same as in the previous paper \cite{huang2010observability} which inspired this experimental section. Note that we used here a normalized version of the NEES, making its swing value equal to 1. We see also that the result is similar for OC-EKF SLAM and ideal EKF SLAM: the NEES varies between 1 and 1.7, in contrast to the EKF SLAM and UKF SLAM which exhibit large inconsistencies over the robot pose We see here that the IEKF remedies inconsistency, with a NEES value that remains close to 1. Note that, it performs here even better than OC-EKF and ideal EKF (whose results are very close to each other), in terms of consistency. The basic difference between IEKF and these filters lies in using or not the current estimate as a linearization point. Uncertainty directions being very dependent from the estimate, what Figure \ref{fig::NEES} suggests is that they may not be correctly captured if computed on a different point.
The other aspect of the evaluation of an EKF-like method is performance: regardless of the relevance of the covariance matrix returned by the filter (i.e. consistency), pure performance can be evaluated through RMS of the heading and position error, whose values over time are displayed in Figure \ref{fig::heading_RMS}. They confirm an expectable result: solving consistency issues  improves the accuracy of the estimate as a byproduct, as wrong covariances yield wrong gains \cite{Bar-Shalom}. 


Selecting a \emph{single run}, we can also illustrate the inconsistency issue in terms of covariance and information. Figure \ref{fig::EKF_enveloppe} displays the heading error for EKF, UKF, IEKF, OC-EKF and ideal EKF SLAM, and the $99 \%$ envelope returned by each filter. This illustrates both the false observability issue and the resulting inconsistency of EKF and UKF: the heading uncertainty is reduced over time while the estimation error goes outside the $99 \%$ envelope. To the opposite, the behavior of IEKF, OC-EKF and ideal EKF is sound. Figure \ref{fig::EKF_enveloppe2} shows the map and the landmarks $99 \%$ uncertainty ellipsoids: Similarly, both EKF and UKF fail capturing the true landmarks 'positions within the $99 \%$ ellipsoids whereas the three over filters succeed to do so. Finally, Figure \ref{fig::info_rot} displays the evolution of the information over a shift corresponding to infinitesimal rotations as defined in Theorem \ref{SLAM:big:thm}, that is  the evolution over time of the quantity $\delta X_n^TP_{n|n}^{-1}\delta X_n$. The theorem is successfully illustrated: the latter quantity is always decreasing for the IEKF, ideal EKF, OC-EKF but not for the EKF and the (slightly better to this respect) UKF.

\section{Conclusion}

This work evidences that the EKF algorithm for SLAM is not inherently inconsistent - at least regarding inconsistency related to unobservable transformations of the global frame - but the choice of the right coordinates for the linearization process is pivotal. We showed that applying the recent theory of the IEKF - an EKF (slight) variant - leads to provable properties regarding observability and consistency. Extensive Monte-Carlo simulations have illustrated the consistency of the new method and the striking improvement over  EKF,  UKF, OC-EKF, and more remarkably over the ideal EKF also, which is the - impossible to implement - variant of the EKF where the system is linearized about the true trajectory.

Note that, the IEKF approach may prove relevant beyond SLAM to some other problems in robotics   as well, such as autonomous navigation (see \cite{barczyk2015invariant}),  and in combination with controllers,   notably for motion planning purposes,  see \cite{diemer2014invariant}. In \cite{barraunolcos} the IEKF has proved to possess \emph{global} asymptotic convergence properties on a simple localization problem of a wheeled robot, which is a strong property. The IEKF has also been patented for navigation with inertial sensors \cite{brevet_alignement}.

Nowadays nonlinear optimization based SLAM algorithms are becoming
popular as compared with EKF SLAM, see e.g. \cite{Dellaert} for one of the first papers on the subject. We yet anticipate  a simple EKF SLAM with consistency properties will prove useful to the research community, the EKF SLAM having been abandoned in part due to its inconsistency.  The general EKF has proved useful in numerous industrial applications, especially in the field of guidance and navigation. It has the benefits of being 1-recursive, avoiding to store the whole trajectory and 2-suited to on-line  real-time applications. Moreover the aerospace and defense industry has developed a corpus of experience for its industrial implementation and validation. And the IEKF is a variant that, being in every respect similar to EKF,  retains all its advantages, but which possesses additional guaranteed properties.  Note also that, all the improvements of the EKF for SLAM such as e.g., the SLAM of \cite{Paz} and sparse extended information filters \cite{Thrun}, can virtually be turned into their invariant counterpart.

The high dimensional optimization formulation of the SLAM problem being prone to local minima,  having an accurate initial value (i.e. a small initial estimation error) is very critical \cite{Zhao}. The IEKF SLAM algorithm proposed in the present paper may thus be advantageously used to initialize those methods in challenging situations. 

Besides, we anticipate  our approach based on symmetries could help improve (at least first order) optimization techniques for SLAM.  To understand why, assume by simplicity the sensors to be noise free. Then, moving a candidate trajectory along unobservable directions will not change the cost function, and an efficient optimization algorithm should account for this. And when a gradient descent algorithm is used, only a first-order expansion of the cost function is considered. Our Lie group approach will allow defining steepest descent directions in a alternative geometric way,   that will ``stick'' to the unobservable directions, and the corresponding update will  move along the (Lie group) state space  in a non-linear yet relevant way.  This issue is left for future work, but a thorough understanding of the interest of the invariant approach for the EKF, is a first step in this direction.

\subsection*{Acknowledgements}The authors would like to thank Cyril Joly for his advice.

\appendix

\section{IEKF theory, applications to 2D and 3D SLAM}

In this section we provide more details on the IEKF theory on matrix Lie groups, and show how the underlying Lie group structure  of the SLAM problem has been used indeed to build the IEKF SLAM Algorithm \ref{algo::IEKF_SLAM}.  We also provide the IEKF equations for 3D SLAM. For more information on the IEKF see \cite{barrau2014invariant} and references therein.

\subsection{Primer on matrix Lie groups}\label{primer}

A matrix Lie group $G$ is a subset of square invertible $\bar N \times \bar N$ matrices  $ \mathcal{M}_{\bar N}(\RR)$ verifying the following properties:
\[
I_{\bar N} \in G, \qquad \forall g \in G, g^{-1} \in G , \qquad \forall a,b \in G, ab \in G
\]
where $I_{\bar N}$ is the identity matrix of $\RR^{\bar N}$. 
If $\gamma(t)$ is a curve over $G$ with $\gamma(0)=I_{\bar N}$, then its derivative at $t=0$ necessarily lies in a subset $\mathfrak{g}$ of  $\mathcal{M}_{\bar N}(\RR)$. $\mathfrak{g}$ is a vector space and it is called the Lie algebra of $G$. It has same dimension $\text{dim }\frak{g}$ as $G$. Thanks to a  {linear} invertible map denoted by  $\mathcal{L}_{\frak{g}}: \RR^{\dim \frak{g}} \rightarrow \frak{g}$,  one can advantageously identify $\mathfrak{g}$ to  $\Rg$. Besides, the vector space $\mathfrak{g}$ can be mapped to the matrix Lie group $G$ through the classical matrix exponential $\exp_m$. Thus, $\RR^{\dim \mathfrak{g}}$ can be mapped to $G$ through the Lie exponential map    defined by $\exp(\xi):=\exp_m(\mathcal{L}_{\mathfrak{g}}(\xi))$ for $\xi\in\Rg$.  This map is   invertible for small $\xi$, and we have $(\exp(\xi))^{-1}=\exp(-\xi)$. The well-known Baker-Campbell-Hausdorff (BCH) formula gives a  series expansion for the product $\exp(\xi)\exp(\zeta)$. In particular it ensures $\exp(\xi)\exp(\zeta)=\exp(\xi+\zeta+T)$, where $T$ is of the order $O(\norm{\xi}^2,\norm{\zeta}^2,\norm{\xi}\norm{\zeta})$.  For any $g\in G$, the adjoint matrix $Ad_g\in\RR^{\text{dim } \frak g\times \text{dim } \frak g}$ is defined by $g\exp(\zeta)g^{-1}=\exp(Ad_g \zeta)$ for all $\zeta\in \mathfrak{g}$.  We now give explicit formulas for two groups of particular interest for the SLAM problem. 

\subsubsection{Group of direct planar isometries $SE(2)$}
\label{sect::tuto_SE2}
 This famous group in robotics can be defined using homogeneous matrices, i.e.,  $
G=SE(2):= \left\lbrace g= \begin{pmatrix} R(\theta) & x \\ 0_{1,2} & 1 \end{pmatrix};\theta \in\RR, ~x\in\RR^2 \right\rbrace$.  \text{Let}~$\zeta=\begin{pmatrix} \alpha \\ u \end{pmatrix}$, then $  \mathcal{L}_{\mathfrak{se}(2)} (\zeta) =\begin{pmatrix} \alpha J&u\\0_{1\times 2}&0 \end{pmatrix}$, where $\alpha\in\RR$ and $u\in\RR^2$. We have $ { \mathfrak{g}=\frak{se}(2)= \left\lbrace \mathcal{L}_{\mathfrak{se}(2)} (\zeta),\zeta\in\RR^3\right\rbrace}$. The Lie exponential writes $
\exp(\zeta) = \begin{pmatrix}R(\alpha) & B(\alpha) u \\ 0_{1,2} & 1 \end{pmatrix}$ where  $B(\alpha) = \begin{pmatrix} \frac{ \sin \left( \alpha \right) }{ \alpha } & - \frac{ 1 - \cos \left( \alpha \right) }{ \alpha } \\ \frac{1-\cos \left( \alpha \right) }{ \alpha } &  \frac{ \sin \left( \alpha \right) }{ \alpha } \end{pmatrix}$. We have  $Ad_{g}=\begin{pmatrix} 1&0_{1\times 2}\\-Jx&R(\theta) \end{pmatrix}$. 

\subsubsection{Group of multiple direct spatial isometries $SE_{K+1}(2)$}
\label{sect::tuto_SE23}We now introduce a simple extension of $SE(2)$, inspiring from  preliminary remarks in \cite{bonnabel2012symmetries,barrau2014invariant}. For $\theta\in\RR$ and $x,p^1,\cdots,p^K\in\RR^2$, consider the map  $\Psi:\RR^{2K+3}\to \mathcal M_{K+3}(\RR)$ defined by
\begin{equation}\label{Psi:map}
\Psi: (\theta,x,p^1,\cdots,p^K) \mapsto\left(
\begin{array}{c|c}
  R(\theta) & x~  p^1~\cdots~p^n \\ \hline
  0_{1,2}  & \raisebox{-15pt}{{\mbox{{$I_{K+1}$}}}} \\[-4ex]
  \vdots & \\[-0.5ex]
  0_{1,2}  &
\end{array}
\right)
\end{equation}
and let $G\subset \RR^{(K+3)\times (K+3)}$ be defined by
$$G=\left\lbrace  
\Psi(\theta,x,p^1,\cdots,p^K);  \theta \in \RR, ~ x,p^1,\cdots,p^K \in \RR^2  \right\rbrace$$and denote it by  $G=SE_{K+1}(2)$. Note that,  we recover $SE(2)$ for $K=1$, i.e.,  $SE(2)=SE_1(2)$. 
Letting $\alpha\in\RR$ and 
$\zeta=\begin{pmatrix}\alpha\\x\\p^1\\\vdots\\p^K\end{pmatrix}$ yields $\mathcal{L}_{\frak{se}_{K+1}(2)} (\zeta)=\left(
\begin{array}{c|c}
  \alpha J & x~  p^1~\cdots~p^n \\ \hline
  0_{1,2}  & \raisebox{-15pt}{{\mbox{{$0_{(K+1)\times(K+1)}$}}}} \\[-4ex]
  \vdots & \\[-0.5ex]
  0_{1,2}  &
\end{array}
\right)$
and  $\frak g=\frak{se}_{K+1}(2)= \left\lbrace  \mathcal{L}_{\frak{se}_{K+1}(2)} (\zeta);  \zeta\in\RR^{2K+3}  \right\rbrace$. It turns out, by extension of the $SE(2)$ results, that there exists a closed form for the Lie exponential $\exp=\exp_m\circ \mathcal{L}_{\frak{se}_{K+1}(2)}$ that writes   $\exp(\zeta)=\Psi(\alpha,B( \alpha) x,B( \alpha)p^1,\cdots,B( \alpha)p^K))$ with $B(\alpha) = \begin{pmatrix} \frac{ \sin \left( \alpha \right) }{ \alpha } & - \frac{ 1 - \cos \left( \alpha \right) }{ \alpha } \\ \frac{1-\cos \left( \alpha \right) }{ \alpha } &  \frac{ \sin \left( \alpha \right) }{ \alpha } \end{pmatrix}$. The  $Ad$ is also easily derived by extension of $SE(2)$, but to save space, we only display it once:  $Ad_{\Psi(\hat \theta,\hat x,\hat p^1,\cdots,\hat p^K)}$ is defined as the matrix $G_n$  of eq \eqref{eq::A_H_G_IEKF}.

\subsection{Statement of the general IEKF equations}\label{gen:IEKF}

This section is a summary of the IEKF methodology of \cite{barrau2014invariant,barrau2013intrinsic}. Let $G$ be a matrix Lie group. Consider a general dynamical system $\chi_n\in G\subset\RR^{{\bar N}\times {\bar N}}$ on the group, associated to a sequence of observations $(Y_n)_{n \geqslant 0}\in\RR^p$, with equations as follows :
\begin{equation}
\label{eq::general_dynamical_system_Lie}
\chi_{n} = \chi_{n-1}\exp(w_n)\Gamma_n
\end{equation}
\begin{equation}\label{eq::general_dynamical_system_output_Lie}
Y_n = h(\chi_n)+V_n,
\end{equation}
where $\Gamma_n\in G$ is an input matrix which encodes the displacement according to the evolution model,  $w_n\in \Rg$ is a vector encoding the  model  noise, $h:G\to\RR^p$ is the observation function and $V_n\in\RR^p$ the measurement noise.

The IEKF propagates an estimate obtained after the previous observation $Y_{n-1}$ through the deterministic part of \eqref{eq::general_dynamical_system_Lie}:
\begin{equation}\label{eq::obs_gen_Lie}
\hat \chi_{n|n-1} = \hat\chi_{n-1|n-1} \Gamma_n\end{equation}
To update $\hat\chi_{n |n-1}$ using the new observation $Y_n$, one has to consider an estimation error that is \emph{well-defined} on the group. In this paper we will use the following right-invariant errors 
\begin{equation}\label{eq::right-error}
\eta_{n-1|n-1}:=\chi_{n-1}\hat\chi_{n-1|n-1}^{-1},\quad \eta_{n|n-1}:=\chi_{n}\hat\chi_{n|n-1}^{-1}
\end{equation} which are equal to $I_{\bar N}$ when $\hat\chi =\chi$. The terminology stems from the fact they are invariant to right multiplications, that is, transformations of the form $\hat\chi_{n|n-1}\mapsto \hat\chi_{n|n-1}g,~\chi_n\mapsto \chi_n g$ with $g\in G$. Note that, one could alternatively consider   left-invariant errors but it turns out to be  less fruitful for  SLAM.

\subsubsection{Linearized error equations over the group $G$}The IEKF update   is based upon a first-order expansion of the non-linear system associated to the errors \eqref{eq::right-error} around $I_{\bar N}$.  First, compute the full error's evolution  
\begin{align*}\eta_{n|n-1}&=\chi_{n}\hat\chi_{n|n-1}^{-1}\\&=\chi_{n-1}\exp(w_n)\Gamma_n\Gamma_n^{-1}\hat\chi_{n-1|n-1}^{-1}\\&=\chi_{n-1}\exp(w_n)\hat\chi_{n-1|n-1}^{-1}\\&= \chi_{n-1}\hat\chi_{n-1|n-1}^{-1}\hat\chi_{n-1|n-1}\exp(w_n)\hat\chi_{n-1|n-1}^{-1}\\&=\eta_{n-1|n-1}\exp(Ad_{\hat\chi_{n-1|n-1}}w_n)\end{align*}Note that the term $\Gamma$ has disappeared ! This is a key property for the successes of the invariant filtering approach \cite{bonnabel2008symmetry,barrau2014invariant}. 
  To linearize this equation we define $\xi_{n-1|n-1},\xi_{n|n-1}\in\Rg$ around $I_{\bar N}$ through
  \begin{equation}\label{xxi} 
  \eta_{n-1|n-1}=\exp(\xi_{n-1|n-1}),\quad  \eta_{n|n-1}=\exp(\xi_{n|n-1}),
  \end{equation}As in the standard non-additive noise EKF methodology \cite{Stengel} all terms of order $\norm{\xi_{n-1|n-1} }^2, \norm{w_n}^2,\norm{w_n}\norm{\xi_{n-1|n-1}}$, are assumed small and are neglected. Using the BCH formula, and neglecting the latter terms, we get 
  $$
  \exp(\xi_{n|n-1})=\exp(\xi_{n-1|n-1}+Ad_{\hat\chi_{n-1|n-1}}w_n)
  $$ 
 Using the local  invertibility  of  $\exp$   around $0_{\text{dim }\frak g\times 1}$, we get the following linearized error evolution in $\Rg$: 
\begin{equation}\label{def:A_non_linear_Lie}
\xi_{n|n-1}  =F_n\xi_{n-1|n-1} +G_n  w_n,
\end{equation}
where $F_n=I_{\text{dim }\frak g}$ and  $G_n=Ad_{\hat\chi_{n-1|n-1}}$.

To linearize  the output error,  we now slightly adapt the IEKF theory \cite{barrau2014invariant}   to account for the general form of output \eqref{eq::bicycle_obs}. Note that, $Y_n - h(\hat \chi_{n|n-1}) =h( \chi_{n})-h(\hat \chi_{n|n-1})+V_n=h(\exp( \xi_{n|n-1})\hat \chi_{n|n-1})-h(\hat \chi_{n|n-1})+V_n$. As $\xi_{n|n-1}$ is assumed small, and as $\exp(0_{\text{dim }\frak g\times 1})=I_{\bar N}$, a first-order  Taylor expansion in   $\xi\in\Rg$  arbitrary, allows definit $H_n$ as follows 
\begin{equation}
\label{def:H_non_linear_Lie}
\begin{aligned}
h( \exp(\xi)\hat \chi_{n|n-1})- h(\hat \chi_{n|n-1}) := H_n  \xi+O(\norm{\xi}^2)
\end{aligned}
\end{equation}

\subsubsection{Computing the Kalman gain} The gain is   computed  as for the standard EKF \cite{Stengel}, but where the linearized error system to be considered is  \eqref{def:A_non_linear_Lie}-\eqref{def:H_non_linear_Lie}. 
\subsubsection{Update}
As in the standard theory, the Kalman gain matrix $K_n$ allows computing an estimate of the linearized error after the observation $Y_n$ through $\xi_{n|n}= K_nz_n$, where $z_n =Y_n - h(\hat \chi_{n|n-1})$. Recall the state estimation errors  defined by \eqref{eq::right-error}-\eqref{xxi} are of the form $\exp(\xi)= \chi\hat \chi^{-1}$, that is, $\chi=\exp(\xi)\hat\chi$.  Thus  an estimate of  $\chi_n$ after observation $Y_n$ which is consistent with \eqref{eq::right-error}-\eqref{xxi}, is obtained through the following  Lie group counterpart of the linear update \eqref{eq::update_linear} 
\begin{equation}\label{exp:eeq}
\hat \chi_{n|n}=\exp(K_nz_n)\hat \chi_{n|n-1}
\end{equation}
The equations of the filter are detailed in Algorithm \ref{algo::IEKF}.

\begin{algorithmic}
\begin{algorithm}
\caption{Invariant Extended Kalman Filter (IEKF)}
Choose initial  $\hat \chi_{0|0}\in G$ and $P_{0|0}\in\RR^{\dim \frak g\times \dim \frak g}=\Cov(\xi_{0|0})$
\label{algo::IEKF}
\LOOP
\STATE Define $H_n$ as in   \eqref{def:H_non_linear_Lie} and let $F_n=I_{\dim \frak g}$ and $G_n=Ad_{\hat\chi_{n|n-1}}$.
\STATE Define $Q_n$ as $\Cov(w_n)$ and $R_n$ as $\Cov(V_n)$.
\STATE \textbf{Propagation} 
\STATE $\hat \chi_{n|n-1} = \hat\chi_{n-1|n-1} \exp(u_n)$
\STATE $P_{n|n-1} =  F_nP_{n-1|n-1} F_n^{-1}  +   G_nQ_nG_n^T  $
\STATE \textbf{Update} 
\STATE $  z_n=Y_n- h \left( \hat \chi_{n|n-1} \right)$
\STATE $  S_n = H_n  P_{n|n-1}H_n^T + R_n $,
 \STATE $
  K_n = P_{n|n-1} H_n^T S_n^{-1}$
  \STATE $P_{n|n} = [I-K_n H_n]P_{n|n-1} $
\STATE{ $ \hat \chi_{n|n} = \exp(K_nz_n)\hat \chi_{n|n-1}$}
\ENDLOOP
\end{algorithm}
\end{algorithmic}

\subsection{Lie group based derivation of the 2D IEKF-SLAM}
\label{sect::IEKF_SLAM_details}

In this section we show step by step the IEKF-SLAM Algorithm \ref{algo::IEKF_SLAM} is a strict application of Algorithm \ref{algo::IEKF}.
\subsubsection{Underlying Lie group} The Lie group that underlies the SLAM problem, is $G=SE_{K+1}(2)$ introduced in Appendix  \ref{sect::tuto_SE23}. Let us  apply the general theory of the IEKF to this group. To define the Lie group counterpart $\chi_n$ of the state $X_n$ defined by  \eqref{state:def}, we let $\chi_n=\Psi(X_n)$.  The model equations \eqref{eq::bicycle_dyn} write 
$$\chi_n=\chi_{n-1}\tilde\Gamma_n
$$
with $\tilde\Gamma_n=\Psi(\omega_n+w_n^\omega, (v_n+w_n^v)^T,0_{1\times 2}\cdots,0_{1\times 2})^T$. At the propagation step, the IEKF propagates the estimate through the corresponding deterministic equations
$$\hat\chi_{n|n-1}=\hat\chi_{n-1|n-1}\Gamma_n
$$with $\Gamma_n=\Psi(\omega_n, v_n^T,0_{1\times 2}\cdots,0_{1\times 2})^T$.  
\subsubsection{Right-invariant error \eqref{eq::right-error}}  A simple matrix multiplication shows that $\eta_{n|n-1}=\chi_n\hat\chi_{n|n-1}^{-1}=\Psi(e_{n|n-1})$, where
\begin{equation}
\label{eq::non_linear_error}
e_{n|n-1} = \begin{pmatrix}
\theta_n - \hat \theta_{n|n-1} \\
x_n-  [R \left(  \theta_n -{\hat \theta}_{n|n-1}\right) \hat x_{n|n-1}]  \\
  p^1_n-[R \left(  \theta_n -{\hat \theta}_{n|n-1}\right)   \hat p^1_{n|n-1}]
  \\
 \vdots \\
  p^K_n- [R \left(  \theta_n -{\hat \theta}_{n|n-1}\right)    \hat p^K_{n|n-1}]
\end{pmatrix}
\end{equation}and where $e_{n|n}$ is defined analogously.
\subsubsection{Linearized error}\label{lnz:sec} The linearized error is defined by \eqref{xxi}, that is, $\eta_{n|n-1}=(\exp(\xi_{n|n-1}))$. As terms of order $O(\norm{\xi_{n|n-1}}^2)$ are to be neglected in the linearized equations, it suffices to compute a first-order approximation of ${\xi_{n|n-1}}$. First note that 
$\xi_{n|n-1}$ defined by \eqref{xi::non_linear_error} is a linear approximation to $e_{n|n-1}$, that is, $e_{n|n-1}=\xi_{n|n-1}+O(\norm{\xi_{n|n-1}}^2)$. This readily implies
$$
 I_{2K+3}+\mathcal{L}_{\frak{se}_{K+1}(2)}(\xi_{n|n-1})+O(\norm{\xi_{n|n-1}}^2)=\Psi(e_{n|n-1})$$Recalling $ I_{2K+3}+\mathcal{L}_{\frak{se}_{K+1}(2)}(\xi)+O(\norm{\xi}^2)=\exp_m{(\mathcal{L}_{\frak{se}_{K+1}(2)}(\xi))}:=\exp(\xi)$ we see \eqref{xi::non_linear_error}  is a first approximation of ${\xi_{n|n-1}}$ as defined in \eqref{xxi} indeed. 
\subsubsection{Linearized error propagation} 
The way the linearized error \eqref{xi::non_linear_error} propagates has already been computed and consists of \eqref{first:linearized}, which is the same as the first equation of \eqref{def:A_non_linear_Lie_2}. It exactly  matches  what can be expected from the general theory, that is, equation 
\eqref{def:A_non_linear_Lie}, recalling that $G_n$ of \eqref{eq::A_H_G_IEKF} is the map $Ad_{\hat\chi_{n-1|n-1}}$ of the group $G=SE_{K+1}(2)$ indeed. 
\subsubsection{Linearized output map}
Note the general definition of $H_n$ of \eqref{def:H_non_linear_Lie} here boils down to the one of  Section \ref{Sec:22} as here $\hat\chi=\Psi(\hat X)$, $\exp(\xi)\hat\chi=\chi=\Psi(X_n)$ and $\xi$ is given by \eqref{xi::non_linear_error}.
\subsubsection{Estimate update}Referring to Appendix \ref{sect::tuto_SE23} and the definition of exponential map of $G=SE_{K+1}(2)$, a simple matrix multiplication shows  that 
$\exp(\xi)\Psi(\hat X)=\varphi(\xi,\hat X)$ with $\varphi$ defined by \eqref{eq::exp_SE2} indeed.

\subsection{Equations of the IEKF-SLAM in 3D }

Extending the group $SE_{K+1}(2)$ to the 3D case, and applying the general IEKF theory of Section \ref{gen:IEKF} , we  derive in the present section an IEKF for 3D SLAM. Due to space limitations an as it is not the primary object of the present paper we pursue extreme brevity of exposure. See also \cite{barrau2013intrinsic,barrau2014invariant}. Note that, although the 3D SLAM equations make use of rotation matrices,   they are in fact totally intrinsic:  When using quaternions (recommended) or Euler angles (not recommended) they write the same as the group $SE_{K+1}(3)$ we introduce does in fact not depend on a specific representation of rotations. 

\subsubsection{3D SLAM model}

The equations of the robot in 3D and in continuous time write:
\begin{equation}
\label{eq::bicycle_dyn_3D-ct}
\begin{aligned}
 \dotex R_t & = R_{t} (\omega_t+w_t^\omega)_\times, \quad 
 \dotex x_t  =R_{t} (v_t+ w_t^v),\\
 \dotex p_t^j & = p_t^j,\quad 1\leq j\leq K
\end{aligned}
\end{equation}
where $R_t\in SO(3)$ is a rotation matrix that represents the robot's orientation at time $t$, $\omega_t\in \RR^3$ denotes the angular velocity of the robot measured by a gyrometer or by odometry (in combination with a unicycle model for a terrestrial vehicle),  $v_t \in \mathbb{R}^3$ the  velocity in the robot's frame, and $p_t^j\in \mathbb{R}^3$ is the position of landmark $j$, and  where  $(b)_\times$ for $b\in\RR^3$ denotes the skew symmetric matrix of $\RR^{3\times 3}$ such that for any $x\in\RR^3$ we have $(b)_\times x=b\times x$. Finally $w_t^\omega$ and $w_t^v$ denote (resp.) the noise on angular and linear velocities. Although the theory of IEKF could very well be applied directly to  this continuous time dynamics as in  \cite{barrau2014invariant}, we apply it here to a discretized model, to be consistent with the rest of the article. Although exact discretization of the noisy model on the group is beyond reach \cite{barrau2013intrinsic}, letting  $\Delta t$ be the time step, the following first-order integration scheme  is widely used:
\begin{equation}
\label{eq::bicycle_dyn_3D}
\begin{aligned}
R_n & = R_{n-1} \exp_m[(w_n^\omega)_\times]\Omega_n, \quad
 x_n  =x_{n-1}+R_{n-1} (v_n+ w_n^v), \\
 p_n^j & = p_{n-1}^j,\quad 1\leq j\leq K
\end{aligned}
\end{equation}
where the increments $\Omega_n,v_n$ are obtained solving the noise-free initial conditions during the $n$-th time step with initial condition $R=I_3,x=0$,  and where the following discrete noise
\begin{equation}
\label{eq::Qn_3D}
w_n^T:=( (w_n^{\omega})^T,(\ w_n^v)^T,0_{1\times 3K})^T,\quad Q_n=  \mathbb{E}(w_nw_n^T)
\end{equation}
is obtained by integration of the corresponding white noises. Note that, this scheme  is accurate  to first-order terms in $\Delta t$.  
 A general landmark observation in the car's frame reads:
\begin{equation}
\label{eq::bicycle_obs_3D}
Y_n = \begin{pmatrix} 
\tilde h \left[ R_n^T \left( p^1-x_n \right) \right]+V_n^1 \\
\vdots \\
\tilde h \left[ R_n^T \left( p^K-x_n \right) \right]+V_n^K 
\end{pmatrix}
\end{equation}
where $Y_n \in \RR^{3 K}$ (or $\RR^{2K}$ for monocular visual SLAM) is the observation of the features at time step $n$,  and $V_n$ the observation noise. We let the output noise covariance matrix be $\tilde R_n\in\RR^{3K\times 3K}$ (not to be confused with the rotation $R_n$). 

\subsubsection{Underlying Lie group}  The Lie group that underlies the problem is the group $G=SE_{K+1}(3)$ that we introduce as follows. For $R\in SO(3)$  and $x,p^1,\cdots,p^K\in\RR^3$   let 
\begin{equation}\label{Psi:map_3D}
\Psi: (R,x,p^1,\cdots,p^K) \mapsto\left(
\begin{array}{c|c}
  R & x~  p^1~\cdots~p^n \\ \hline
  0_{1,3}  & \raisebox{-15pt}{{\mbox{{$I_{K+1}$}}}} \\[-4ex]
  \vdots & \\[-0.5ex]
  0_{1,3}  &
\end{array}
\right)
\end{equation}
and let  $G\subset \RR^{(K+4)\times (K+4)}$ be defined as 
$$G=\left\lbrace \Psi(R,x,p^1,\cdots,p^K) ; R\in SO(3), ~ x,p_1,\cdots,p^K \in \RR^3  \right\rbrace$$and denote it by  $G=SE_{K+1}(3)$.   We then have $\mathcal{L}_{\frak{se}_{K+1}(2)} (\begin{pmatrix}\omega\\x\\p^1\\\vdots\\p^K\end{pmatrix})=\left(
\begin{array}{c|c}
  (\omega)_\times & x~  p^1~\cdots~p^n \\ \hline
  0_{1,3}  & \raisebox{-15pt}{{\mbox{{$0_{(K+1)\times(K+1)}$}}}} \\[-4ex]
  \vdots & \\[-0.5ex]
  0_{1,3}  &
\end{array}
\right)$
and  $\frak g=\frak{se}_{K+1}(3)= \left\lbrace  \mathcal{L}_{\frak{se}_{K+1}(3)} (\zeta);  \zeta\in\RR^{3K+6}  \right\rbrace$.  For $\zeta\in\RR^{3K+6} $, by extension of the $SE(3)$ results, we have the closed form:
\begin{align}\label{exp:def:3D}\exp (\zeta)= I_{K+4} + S + \frac{1 - \cos(||\zeta||)}{||\zeta||^2} S^2 + \frac{ ||\zeta|| -\sin(||\zeta||}{||\zeta||^3} S^3\end{align} where  $S =\mathcal{L}_{\frak{se}_{K+1}(3)} (\zeta)$. As easily seen by analogy with $SE(3)$\begin{align}\label{ad:def:3D}Ad_{\Psi(R,x,p^1,\cdots,p^K)} =\left(
\begin{array}{c|c}
  R& 0_{3\times 3}~  \cdots~0_{3\times 3}\\ \hline
  (x)_\times R  & \raisebox{-15pt}{{\mbox{{$\begin{matrix}&&\\R& &\\&\ddots&\\&&R\end{matrix}$}}}} \\[-5ex]
  \vdots & \\[-0.5ex]
  (p^K)_\times R  &
\end{array}
\right)\end{align}
\subsubsection{Link with the dynamical model}Let the state be $X=(R,x,p_1,\cdots,p_K)$, and let $\chi\in G$ be $\chi=\Psi(X)$, and let $\hat X$ and $\hat \chi$ be their estimated counterparts. It is easily seen that up to terms that will disappear in the linearization process anyway, the model  \eqref{eq::bicycle_dyn_3D} for the state is mapped through $\Psi$ defined at \eqref{Psi:map_3D}, to a model of the form  \eqref{eq::general_dynamical_system_Lie}.
\subsubsection{Right-invariant error \eqref{eq::right-error}} It writes $\eta:=\chi\hat\chi^{-1}=\Psi(R\hat R^{T},x-R\hat R^{T}\hat x,\cdots,p^K-R\hat R^{T}\hat p^K)$.
\subsubsection{Linearized error}   Using the matrix logarithm, define $\tilde\omega\in\RR^3$ as the solution of $\exp_m[(\tilde\omega)_\times] =R\hat R^{T}$. Neglecting terms of order $O(\norm{\tilde\omega}^2)$, we have $x-R\hat R^{T}\hat x=x-\hat x-\tilde\omega \times\hat x$. A first order identification   as in Appendix \ref{lnz:sec}    thus yields $\xi=(\tilde\omega,x-\hat x-\tilde\omega \times\hat x,\cdots,p^K-\hat p^K-\tilde\omega \times\hat p^K)$ as a vector that satisfies the definition $\exp(\xi)=\eta$ up to terms of order $O(\norm{\xi}^2$). 
\subsubsection{Linearized error propagation}As in the standard EKF theory, 
the IEKF propagates an estimate obtained after the previous observation $Y_{n-1}$ through the deterministic part of \eqref{eq::bicycle_dyn_3D}, or equivalently \eqref{eq::general_dynamical_system_Lie} in matrix form. Thus, the propagation equation is  given by \eqref{def:A_non_linear_Lie} where $F_n=I_{3K+6}$ and $G_nw_n=Ad_{\hat \chi_{n|n-1}}(w_n)$ as a direct application of the theory.
\subsubsection{Linearized output map} 
To apply Definition \eqref{def:H_non_linear_Lie}, simply note that $R^Tp-\hat R^T\hat p=R^T(p-R\hat R^T\hat p)\approx \hat R^T(p-R\hat R^T\hat p)$. Thus
\begin{equation}
\label{eq::A_H}
H_n  = \begin{pmatrix}
\nabla \tilde h^1 \cdot R \left( \hat \theta_{n|n-1} \right)^T \begin{pmatrix} 0_{3,1} & -I_3 &   I_3 &0_{3,3(K-1)} \end{pmatrix} \\
\vdots \\
\nabla \tilde h^K \cdot R \left( \hat \theta_{n|n-1} \right)^T \begin{pmatrix} 0_{3,1} & -I_3 & 0_{3,3(K-1)} & I_3 \end{pmatrix}
\end{pmatrix},
\end{equation}
where $\nabla \tilde h^k$ is the Jacobian of $\tilde h$ computed at $\hat R_{n|n-1}^T \left( \hat p_{n|n-1}^k - \hat x_{n|n-1} \right)$.

The various steps are gathered in Algorithm \ref{algo::IEKF_SLAM83D}.

\begin{algorithmic}
\begin{algorithm}
\caption{IEKF SLAM: the 3D case}
\label{algo::IEKF_SLAM83D}
\STATE{The state is  defined by $X=(R,x,p^1,\cdots,p^K)\in SO(3)\times\RR^{3+3K}$ and the model is \eqref{eq::bicycle_dyn_3D}-\eqref{eq::bicycle_obs_3D}. Pick an initial estimate $\hat X_0$ with covariance matrix $P_0$.}
\LOOP
\STATE Let $F_n=I_{3K+6}$,  $H_n$ as in \eqref{eq::A_H}, $G_n=Ad_{\psi(\hat X_{n-1|n-1})}$ using \eqref{ad:def:3D}
\STATE Define $Q_n$ by  \eqref{eq::Qn_3D}. $\tilde R_n$ is the observation noise cov. matrix.
\STATE \textbf{Propagation} 
\STATE{ $\hat R_{n|n-1} =\hat  R_{n-1|n-1} \Omega_n$}
\STATE{ $\hat  x_{n|n-1}  =\hat  x_{n-1|n-1}+\hat R_{n-1|n-1} v_n$}
\STATE{ $\hat p_{n|n-1}^j  = \hat p_{n-1|n-1}^j$} for all $1\leq j\leq K$
\STATE $P_{n|n-1} = F_n P_{n-1|n-1} F_n^T + G_n Q_n G_n^T$
\STATE \textbf{Update} 
\STATE $  z_n=Y_n- \begin{pmatrix} 
\tilde h \left[ R_{n|n-1}^T \left( \hat p^1_{n|n-1}- \hat x_{n|n-1} \right) \right] \\
\vdots \\
\tilde h \left[ R_{n|n-1}^T \left( \hat p^K_{n|n-1}- \hat x_{n|n-1} \right) \right] 
\end{pmatrix}$
\STATE $  S_n = H_n P_{n|n-1} H_n^T + \tilde R_n $,
 \STATE $
  K_n = P_{n|n-1} H_n^T S_n^{-1}$
  \STATE $P_{n|n} = [I-K_n H_n]P_{n|n-1} $
\STATE{Compute $\hat \chi_{n|n}=\exp(K_nz_n)\Psi(\hat X_{n|n-1} )$ using \eqref{exp:def:3D} and let $\hat X_{n|n}=\Psi^{-1 }(\hat \chi_{n|n})$. }

\ENDLOOP
\end{algorithm}
\end{algorithmic}


\end{document}